\documentclass{article}

\usepackage[preprint]{neurips_2021}

\usepackage[utf8]{inputenc} 
\usepackage[T1]{fontenc}    
\usepackage{hyperref}       
\usepackage{url}            
\usepackage{booktabs}       
\usepackage{amsfonts}       
\usepackage{nicefrac}       
\usepackage{microtype}      
\usepackage{physics}
\usepackage{amsmath}
\usepackage{amssymb}
\usepackage{mathtools}
\usepackage[mathscr]{eucal}
\usepackage{color}
\usepackage{bbm}
\usepackage{amsthm}
\usepackage{mathrsfs}
\usepackage{subfigure}

\def\w{\mathbf{w}}
\def\thet{\mathbf{\theta}}

\def\E{\mathbb{E}}
\def\R{\mathbb{R}}
\def\V{\mathbb{V}}
\def\X{\mathbf{X}}
\def\A{\mathbf{A}}
\def\D{\mathbf{D}}
\def\d{\mathbf{d}}
\def\b{\mathbf{b}}
\def\k{\mathbf{k}}
\def\P{\mathbf{P}}
\def\I{\mathbf{I}}

\def\R{\mathbb{R}}
\def\X{\mathbf{X}}
\def\Y{\mathbf{Y}}
\def\L{\mathbf{L}}
\def\v{\mathbf{v}}
\def\r{\mathbf{r}}

\def\bPhi{\mathbf{\Phi}}

\def\bGamma{\mathbf{\Gamma}}
\def\bPhi{\mathbf{\Phi}}

\def\calS{\mathcal{S}}
\def\calA{\mathcal{A}}
\def\calP{\mathcal{P}}
\def\y{\mathbf{y}}

\newcommand{\model}{X-ETD($n$)}
\newcommand{\notlearned}{Monte Carlo}

\mathtoolsset{showonlyrefs}

\newcounter{ass_counter}
\newtheorem{assumption}[ass_counter]{Assumption}

\newcounter{prop_counter}
\newtheorem{proposition}[prop_counter]{Proposition}

\newcounter{lem_counter}
\newtheorem{lemma}[lem_counter]{Lemma}

\newcounter{thm_counter}
\newtheorem{theorem}[thm_counter]{Theorem}

\usepackage[dvipsnames]{xcolor}
\usepackage{tikz}
\usetikzlibrary{arrows,shapes,fit,shadows,decorations.pathmorphing}
\usepackage{subfigure,wrapfig}

\tikzstyle{circ}=[circle,draw=black,thick,minimum size=8mm,>=stealth] 
\tikzstyle{state}=[circle,draw=black,minimum size=2em,>=stealth] 
\tikzstyle{ellip}=[ellipse,draw=black,thick,minimum size=4mm, text width=.4cm]
\tikzstyle{system} = [draw, dotted, minimum height=2em]
\tikzstyle{input} = [text centered, minimum height=2em]
\tikzstyle{remark} = [text centered, minimum height=2em, YellowOrange]
\tikzstyle{small_txt} = [text centered, minimum height=.4em]
\tikzstyle{null} = [inner sep=0, outer sep=0]
\tikzstyle{op} = [draw, fill=lightgray!10, text centered,
                  minimum height=2em, rounded corners]
\tikzstyle{mlp} = [draw, text width=10em, fill=lightgray!10, text centered,
                  minimum height=2em, rounded corners]
\tikzstyle{top} = [draw, fill=RoyalBlue!50, text centered,
                  minimum height=2em, rounded corners]
\tikzstyle{inp} = [text centered, minimum height=1em]
\tikzstyle{tinp} = [text centered, minimum height=1em, RoyalBlue]
\tikzstyle{dgraph}=[->, line width=1pt]

\title{Learning Expected Emphatic Traces for Deep RL}

\author{Ray Jiang \\ DeepMind \\ London, UK \\ \texttt{rayjiang@google.com} \And 
        Shangtong Zhang \\ University of Oxford \\ Oxford, UK \\ \texttt{shangtong.zhang@cs.ox.ac.uk} \And 
        Veronica Chelu \\ McGill University \\ Montreal, QC, Canada \\ \texttt{veronica.chelu@mail.mcgill.ca} \And  
        Adam White \\ DeepMind \\ Edmonton, Canada \\ \texttt{adamwhite@google.com} \And
        Hado van Hasselt \\ DeepMind \\ London, UK \\ \texttt{hado@google.com}}
\begin{document}

\maketitle

\begin{abstract}

Off-policy sampling and experience replay are key for improving sample efficiency and scaling model-free temporal difference learning methods. When combined with function approximation, such as neural networks, this combination is known as the deadly triad and is potentially unstable. Recently, it has been shown that stability and good performance at scale can be achieved by combining emphatic weightings and multi-step updates. This approach, however, is generally limited to sampling complete trajectories in order, to compute the required emphatic weighting. In this paper we investigate how to combine emphatic weightings with non-sequential, off-line data sampled from a replay buffer. We develop a multi-step emphatic weighting that can be combined with replay, and a time-reversed $n$-step TD learning algorithm to learn the required emphatic weighting. We show that these state weightings reduce variance compared with prior approaches, while providing convergence guarantees. We tested the approach at scale on Atari 2600 video games, and observed that the new X-ETD($n$) agent improved over baseline agents, highlighting both the scalability and broad applicability of our approach.
\end{abstract}

Many deep reinforcement learning systems are not sample efficient. A simple and effective way to improve sample efficiency is to make better use of prior experience via replay \citep{lin1992self, mnih2015, SchaulQAS15, rainbow}. Previous work demonstrated, somewhat surprisingly, that increasing the amount of replay in a model-free learning system can surpass the sample efficiency and final performance of model-based agents which utilize significantly more computation \citep{Hasselt2019WhenTU}. 

While improving on sample efficiency, using experience replay also introduces more potential for instability. Most approaches update from mini-batches of previous experience corresponding to older policies, and are therefore off-policy \citep{mnih2015, rainbow}. Unfortunately combining bootstrapping via temporal-difference updates, function approximation and off-policy learning---known as the \emph{deadly triad} \citep{sutton2018}---can destabilize learning resulting in ``soft divergence'', slower learning, and reduced sample efficiency even if the parameters do not fully diverge \citep{hasselt2018}. Additionally, learning methods based on off-policy \emph{importance sampling} (IS) corrections can result in high variance and poor performance during learning. This can be improved in practice by bootstrapping more, for instance by cleverly clipping the IS ratios as in the V-trace algorithm \citep{espeholt2018} or ABTD \citep{Mahmood:2017AB}, though bootstrapping too much can exacerbate issues related to the deadly triad.

In order to prevent divergence, we can try to correct the mismatch between the state distribution in the replay buffer and the current policy. The emphatic TD($\lambda$) algorithm  \citep[ETD($\lambda$);][]{sutton2016emphatic} reweights the TD($\lambda$) updates with an ``emphatic'' state weighting based on a ``followon'' trace that, intuitively, keeps track of how important each state is in the learning process.  For instance, states that are heavily used to update other state values, via bootstrapping, will receive more emphasis, which ensures their values are sufficiently accurate even if they are updated infrequently.  This prevents divergent learning dynamics.

ETD($\lambda$) uses eligibility traces and has not yet been combined with neural network function approximation or replay. However, the idea of emphatic weighting is not restricted to trace-based (``backward-view'') algorithms and can be extended to other settings.  For instance, \emph{$n$-step Emphatic TD} (NETD) \citep{jiang2021} is a recent algorithm that combines emphatic weighting with $n$-step forward-view updates as well as V-trace learning targets. For consistency with the canonical name TD($n$), for $n$-step TD learning, we call this algorithm ETD($n$) in this paper. This was shown to outperform V-trace at scale in Atari and diagnostic MDP experiments \citep{jiang2021}.

The emphatic weightings used in ETD($n$) are sequentially accumulated over time, in the form of trajectory-dependent traces, and can thus only be computed from online sequential trajectories, or full episodes of offline trajectories. 
In this paper we investigate how to combine emphatic weightings with non-sequential, off-line data sampled from a replay buffer. The idea is to estimate \emph{expected} emphatic weightings as a function of state \citep{zhang2020provably,van2020expected}, allowing us to appropriately weight the learning updates even if the inputs are sampled out of order. This reduces well-known variance issues with emphatic weightings \citep{ghiassian2018online,imani2018,zhang2020provably}. We show in Sec.~\ref{sec:x_netd_vtrace_theory} that well-estimated emphatic weights reduce the potentially high variance of ETD($n$) and achieve convergence with an upper bound on the bias from the ground-truth value function.

Our contributions include 
1) an off-policy time-reversed TD learning algorithm to learn the expected $n$-step emphatic trace using non-sequential data;
2) a discussion of potential stabilization techniques;
3) an analysis of theoretical properties of variance, stability and convergence for the resulting algorithm {\model};
4) an investigation of practical benefits of the approach when used at scale: we observed improved performance on Atari when using {\model} with replay.

\section{Background}
We denote random variables with uppercase (e.g., $S$) and the obtained values with lowercase letters (e.g., $S = s$). Multi-dimensional functions or vectors are bolded (e.g., $\b$), as are matrices (e.g. $\A$). For all state-dependent functions, we also allow time-dependent shorthands (e.g., $\gamma_t\!=\! \gamma(S_t)$).

\subsection{Reinforcement Learning problem setup}
We consider the usual RL setting in which an agent interacts with an environment, modelled as an infinite horizon \emph{Markov Decision Process} (MDP) $(\calS, \calA, P, r)$, with a finite state space $\calS$, a finite action space $\calA$, a state-transition distribution $P\!:\! \calS \!\times \!\calA \to \calP(\calS)$ (with $\calP(\calS)$ the set of probability distributions on $\calS$ and $P(s^\prime|s, a)$ the probability of transitioning to state $s^\prime$ from $s$ by choosing action $a$), and a reward function $r : \calS \times \calA \rightarrow \R$. A policy $\pi : \calS \rightarrow \calP(\calA)$ maps states to distributions over actions; $\pi(a|s)$ denotes the probability of choosing action $a$ in state $s$ and $\pi(s)$ denotes the probability distribution of actions in state $s$. Let $S_t, A_t, R_t$ denote the random variables of state, action and reward at time $t$, respectively.

The goal of \emph{policy evaluation} is to estimate the \emph{value function} $v_\pi$, defined as the expectation of the discounted return under policy $\pi$:
\begin{align}
G_t & \doteq R_{t+1} + \Sigma_{i=t+1}^\infty \gamma_i R_{i+1} = R_{t+1} + \gamma_{t+1} G_{t+1} \,,\\
v_{\pi}(s) & \doteq \mathbb{E} [ G_t \mid S_t=s, A_{k} \sim \pi(S_{k}), S_{k+1} \sim P(S_{k}, A_{k}) \text{ for all } k\geq t] \,,\label{eq:learning_target}
\end{align}
where $\gamma : \calS \to [0, 1]$ is a discount factor. 
We consider function approximation and use $v_{\w}$ as our estimate of $v_\pi$, where $\w$ are parameters of $v_{\w}$ to be updated.

In the case of \emph{off-policy} policy evaluation,
though our goal is to estimate $v_\pi$,
the actions for interacting with the MDP are sampled according to a different policy $\mu$.
We refer to $\pi$ and $\mu$ as target and behavior policies respectively and make the following assumption for the behavior policy $\mu$:
\begin{assumption}
\label{assu ergodic}
(Ergodicity) The Markov chain induced by $\mu$ is ergodic.
\end{assumption}
\begin{assumption}
\label{assu coverage}
(Coverage) $\pi(a|s) > 0\Longrightarrow \mu(a | s) > 0$ holds for any $(s, a)$.
\end{assumption}
Under Assumption~\ref{assu ergodic},
we use $d_\mu$ to denote the ergodic distribution of the chain induced by $\mu$.
In this paper, 
we consider two off-policy learning settings: the \emph{sequential setting} and the \emph{i.i.d. setting}.
In the sequential setting,
the algorithm is presented with an infinite sequence as induced by the interaction $$(S_0, A_0, R_1, S_1, A_1, R_2, \dots),$$
where $A_t \sim \mu(S_t), R_{t+1} \doteq r(S_t, A_t), S_{t+1} \sim P(S_t, A_t)$.  The idea is then that we update the value and/or policy at each of these states $S_t$, using data following the state (e.g., the sampled return).  Updates at state $S_t$ always happen before updates at states $S_{t+k}$, for $k > 0$.

In the i.i.d. setting, the algorithm is presented with an infinite number of \emph{finite} sequences of length $n$
\[
\{(S^k_0, A^k_0, R^k_1, S^k_1, A^k_1, R^k_2, \dots, S^k_n)\}_{k=1, 2, \ldots} \,,
\]
where the starting state of a sequence is sampled i.i.d., such that $S^k_0 \sim d_\mu$, and then the generating process for the subsequent steps is the same as before: $A^k_t \sim \mu(S^k_t), R^k_{t+1} \doteq r(S^k_t, A^k_t), S^k_{t+1} \sim P(S^k_t, A^k_t)$.
The idea is then that we update the value and/or policy of the first state in each sequence, $S^k_0$, using the rest of that sequence, e.g., by constructing a bootstrapped $n$-step return.

The sequential setting corresponds to the canonical agent-environment interaction \citep{sutton2018}.
Sequential algorithms are often data inefficient, since typically each state $S_t$ is updated only once and then discarded \citep[e.g.,][]{Watkins2004Qlearning}.
One way to improve data efficiency is to store the sequential data in a replay buffer \citep{lin1992self} and reuse these for further updates.
If $\mu$ is stationary and these tuples are uniformly sampled from a large-enough buffer, their distribution is almost the same as $d_\mu$.
Hence uniform replay is akin to the i.i.d. setting.
If we sample from the replay buffer with different priorities, e.g., $S^k_0$ is sampled from some other distribution $d_p$,  the updates to $S^k_0$ can be reweighted with importance-sampling ratios $d_\mu(S^k_0)/d_p(S^k_0)$ to retain the similarity to the i.i.d. setting. Therefore, for simplicity and clarity, we present our theoretical results in the i.i.d. setting.\footnote{In practice, computing $d_\mu(S^k_0)/d_p(S^k_0)$ exactly is usually impossible.
One can, however,
approximate it with $1/(d_p(S^k_0)N)$ with $N$ being the size of the replay buffer.
We refer the reader to \citep{SchaulQAS15} for more details about this approximation.}

\subsection{Policy Evaluation}

We use the sequential setting and linear function approximation to demonstrate three RL algorithms for off-policy policy evaluation. We denote the features of state $S_t$ by $\boldsymbol{\phi}(S_t)$ or $\boldsymbol{\phi}_t$.

\paragraph{Off-policy TD($n$)} 
Off-policy TD($n$) updates $\w$ iteratively as
\begin{align}
\label{eq:tdn}
     \textstyle{\w_{t+1} = \w_t + \alpha \sum_{k=t}^{t+n-1} \left(\prod_{i=t}^{k-1} \rho_i \gamma_{i+1} \right) \, \rho_k \delta_k(\w_t) \boldsymbol{\phi}_t},
\end{align}
where $\rho_t \doteq \frac{\pi(A_t | S_t)}{\mu(A_t | S_t)}$ is an importance sampling (IS) ratio and $\delta_k(\w_t)$ is the TD error:
\begin{align}
\delta_k(\w_t) = R_{k+1} + \gamma_{k+1}v_{\w_t}(S_{k+1})-v_{\w_t}(S_k) \,.
\end{align}

\paragraph{ETD($n$)} Off-policy TD($n$) can possibly diverge with function approximation.
Emphatic weightings are an approach to address this issue \citep{sutton2016emphatic}. In particular, ETD($n$) considers the following ``followon trace'' to stabilize the off-policy TD($n$) updates \citep{jiang2021}:
\begin{align}
\textstyle{F_t = \left(\prod_{i=t-n}^{t-1} \rho_{i} \gamma_{i+1}\right) F_{t-n} + 1, \text{ with }  F_0 = F_1 = \dots = F_{n-1} = 1}, \label{eq:followon_trace}
\end{align}
thus updating $\w_t$ iteratively as
\begin{align}
\label{eq:etdn}
     \textstyle{\w_{t+1} = \w_t + \alpha F_t \sum_{k=t}^{t+n-1} \left(\prod_{i=t}^{k-1} \rho_i \gamma_{i+1} \right) \, \rho_k \delta_k(\w_t) \boldsymbol{\phi}_t}.
\end{align}
\citet{jiang2021} proved that this ETD($n$) update is stable.
In this paper, we consider \emph{stability} in the sense of \citet{sutton2016emphatic},
i.e.,
a stochastic algorithm computing $\qty{\w_t}$ according to 
$\w_{t+1} \doteq \w_t + \alpha_t (\b_t - \A_t \w_t)$ is said to be \emph{stable} if $\A \doteq \lim_{t \to \infty} \E[\A_t]$ is positive definite (p.d.)\footnote{See~\eqref{eq xetdn} for specific forms of vector $b_t$, matrix $A_t$ in our case.}.

Notice $F_t$ in \eqref{eq:followon_trace} is a trace, defined on a sequence of transitions ranging, via recursion on previous values $F_{t-n}$, into the indefinite past. When we do not have access to this sequence to compute the right weighting for a given state, for instance because we sampled this state uniformly from a replay buffer, we need to consider an alternative way to correctly weight the update. \textbf{This incompatibility between the i.i.d. setting and ETD($n$) is the main problem we address in the paper}. To differentiate from our proposed emphatic weighting, from here on we refer to the ETD($n$) trace as the \emph{Monte Carlo} trace since it is a Monte Carlo return in reversed time with ``reward'' signal of $1$ every $n$ steps. In addition, we use the term \emph{emphasis} as a shorthand for ``emphatic trace''.

\section{Proposal: Learn the Expected Emphasis}
\label{sec:learn_emphasis}

In order to apply emphatic traces to non-sequential i.i.d. data, 
we propose, akin to \citet{zhang2020provably}, to directly learn a prediction model that estimates the limiting expected emphatic trace,
i.e., 
we train a function $f_\theta$ parameterized by $\theta$ such that $f_\theta(s)$ approximates $\lim_{t \to \infty} \E_\mu[F_k \mid S_k = s]$.\footnote{For the ease of presentation, we assume the existence of such a limit following \citet{sutton2016emphatic, jiang2021}. The existence can be proved in a similar way to Lemma 1 of \citet{zhang2019}, which we leave for future work.} We use the learned emphasis $f_{\theta_k}$ in place of the {\notlearned} ETD($n$) trace $F_k$, to re-weight the $n$-step TD update in \eqref{eq:etdn}. We refer to the resulting value learning algorithm \emph{eXpected Emphatic TD($n$)}, \emph{\model}.
Thanks to the trace prediction model $f_\theta$,
a sequential trajectory is no longer necessary for computing the emphatic weighting X-ETD($n$).
We know that using a learned expected emphasis can introduce approximation errors.
Thus Section~\ref{sec:x_netd_vtrace_theory} contains a theoretical analysis,
which shows that as long as the function approximation error is not too large, 
stability, convergence, and a reduction in variance, are all guaranteed.
We dedicate this section to the describing how to learn $f_\theta$.

The $n$-step emphatic traces in \eqref{eq:followon_trace} are designed to emphasize  $n$-step TD updates. Consequentially, the trace recursion in \eqref{eq:followon_trace} follows the same blueprint as TD($n$), but in the reverse direction of time. Hence a natural choice of learning algorithm for the expected emphasis is \emph{time-reversed TD learning} that learns the ``reward'' $1$ every $n^{\text{th}}$ step using off-policy TD($n$) with the time index reversed.
Considering the i.i.d. setting,
we update $\theta_k$ using a ``semi-gradient'' \citep{sutton2018} TD update:
\begin{align}
\label{eq:backward_td}
\textstyle{\thet_{k+1} = \thet_k + \alpha_k^{\thet}\left[\left(\prod_{t=1}^{n} \gamma^k_t \rho^k_{t-1} \right) f_{\thet_k}(S^k_0) + 1 - f_{\thet_k}(S^k_n)\right] \nabla_{\thet_k} f_{\thet_k} (S^k_n)},
\end{align}
where $\alpha_k^{\thet}$ is a possibly time-dependent step size,
$\rho^k_t \doteq \frac{\pi(A^k_t | S^k_t)}{\mu(A^k_i | S^k_t)}$,
and $\gamma^k_t \doteq \gamma(S^k_t)$
. 
This update corresponds to the semi-gradient of the following loss,
\begin{align}
\label{eq:learning_f}
\textstyle{\mathcal{L}^F_k \doteq 
\left[\left(\prod_{t=1}^{n} \gamma^k_t\rho^k_{t-1}\right) f_{\thet_k}(S^k_0) + 1 - f_{\theta_k}(S^k_n)\right]^2}.
\end{align}

For the case of linear value function approximation, the update in \eqref{eq:backward_td} may not be stable because its update matrix, $\bPhi^T (\I- (\bGamma \P^T_{\pi})^{n}) \D_{\mu} \bPhi$, is not guaranteed to be positive definite (details in the appendix).
Here $\P_{\pi}$ is the transition matrix such that $\P_\pi(s, s^\prime) \doteq \sum_a \pi(a|s)P(s^{\prime}|s, a)$,  $\bGamma$ is a diagonal matrix such that $\bGamma(s, s) \doteq \gamma(s)$, $\D_\mu$ is a diagonal matrix whose diagonal entry is $d_\mu$,
and $\bPhi$ is the feature matrix whose $s$-th row is $\boldsymbol{\phi}(s)^\top$. 
Thus we propose two stabilization techniques for the time-reversed TD learning updates.

\paragraph{IS Clipping} 
One straightforward way is to clip the IS ratios in \eqref{eq:backward_td} just like in V-trace \citep{espeholt2018},
i.e., we update $\theta$ iteratively as

\begin{align}
\label{eq backward td clip}
\textstyle{\thet_{k+1} = \thet_k + \alpha_k^{\thet}\left[\left(\prod_{t=1}^{n} \gamma^k_t \min(\rho^k_{t-1}, \bar \rho)\right) f_{\thet_k}(S^k_0) + 1 - f_{\thet_k}(S^k_n)\right] \nabla_{\thet_k} f_{\thet_k} (S^k_n)},
\end{align}
for some $\bar{\rho} > 0$; typically $\bar{\rho} = 1$.

Define the substochastic matrix $\P_{\bar \rho}$ such that for any state $s, s^\prime$,
\begin{align}
\textstyle{\P_{\bar \rho}(s, s^\prime) \doteq \sum_a \mu(a | s) \min(\rho(a | s), \bar \rho) p(s^\prime|s, a) \gamma(s^\prime)}.
\end{align}
Then the update matrix of~\eqref{eq backward td clip} is 
$\bPhi^T (\I- (\P_{\bar \rho}^T)^{n}) \D_{\mu} \bPhi$ (see details in the appendix).

We prove that when estimating the expected emphasis using linear function approximation, there exist conditions under which we can guarantee stability at the cost of incurring additional bias.
\begin{proposition}
\label{prop trace clip}
There exists a constant $\tau > 0$ such that the update in \eqref{eq backward td clip} is stable whenever $\bar \rho < \tau$.
\end{proposition}
See its proof in the appendix. One such constant is $\tau = \max_s 1/\gamma(s)$ where the maximum of discounts $\gamma(s)$ is over states. Notice that while achieving stability, clipping at $1/\gamma$ also restricts variance of learning to a finite amount since the Monte Carlo ETD($n$) trace is bounded. In practice, we tune $\bar \rho$ to optimize a bias-stability trade-off. 

\paragraph{Auxiliary Monte-Carlo loss} In most learning settings (e.g., \cite{mnih2015}), both sequential samples and i.i.d. samples are available.
To take advantage of this fact, we can stabilize the emphasis learning by partially regressing on the {\notlearned} emphatic trace. We can thus learn the parameters $\thet$ by TD-learning using samples from the replay buffer and by Monte Carlo learning using online experience:
\begin{align}
\label{eq sequential_iid_update}
\thet_{k+1} = \thet_k &+ \textstyle{ \alpha_k^{\thet}\left[ \left(\prod_{t=1}^{n} \gamma^k_t\rho^k_{t-1}\right) f_{\thet_k}(S^k_{0}) + 1 - f_{\thet_k}(S^k_{n})  \right] \nabla_{\thet_k} f_{\thet_k} (S^k_{n}) }\\
&+ \alpha_k^\theta \beta \left(F_k - f_{\theta_k}(S_k) \right) \nabla_{\theta_k} f_{\theta_k}(S_k).
\end{align}
This update corresponds to the joint loss function:
\begin{align}
\textstyle{\mathcal{L}^{F, \text{MC}}_k \doteq \left[\left(\prod_{t=1}^{n} \gamma^k_t\rho^k_{t-1}\right) f_{\thet_k}(S^k_{0}) + 1 - f_{\theta_k}(S^k_{n})\right]^2 + \beta (f_{\theta_k}(S_k) - F_k)^2},
\end{align}
where $\beta$ is a hyper-parameter for balancing the two losses.
When $f_\theta$ uses linear function approximation, we prove the following guarantee on its stability (proof in the appendix).
\begin{proposition}
\label{prop trace monte carlo}
There exists a constant $\xi$ such that 
the update in \eqref{eq sequential_iid_update} is stable whenever $\beta > \xi$.
\end{proposition}
The time-reversed TD update can be unstable, whereas the Monte Carlo update target $F_k$ can have large variance \citep{sutton2016emphatic, jiang2021}. By choosing $\beta$,
we optimize a variance-stability trade off.

\section{Expected Emphatic TD learning}
\label{sec:x_netd_vtrace_theory}

To prevent deadly triads, we use the learned expected emphasis $f_\thet$ to re-weight the learning updates of TD($n$). In this section, we analyze the resulting algorithm, {\model}.
For simplicity,
let the trace model $f_\theta$ be parameterized by a fixed parameter $\theta$.\footnote{This is not a special setting since the expected emphasis learning process is independent from learning parameters $\w$. } In this section, we analyze {\model} in the sequential setting for the ease of presentation. A similar analysis would apply to {\model} in the i.i.d. setting.
Then {\model} updates $\w$ iteratively as
\begin{align}
\label{eq xetdn}
\w_{t+1} &= \w_t + \alpha_t^{\w} f_{\thet}(S_t) \Delta^\w_t,\\
\text{where } \Delta^\w_t &\doteq \textstyle{\sum_{k=t}^{t+n-1} \left(\prod_{i=t}^{k-1}\gamma_{i+1}\rho_i\right)\rho_k  
(R_{k+1} + \gamma_{k+1}\w_t^\top\boldsymbol{\phi}(S_{k+1})-\w_t^\top\boldsymbol{\phi}(S_t)) \boldsymbol{\phi}(S_t)}.\notag
\end{align}
Equivalently, we can write~\eqref{eq xetdn} as
\begin{align}
\w_{t+1} &= \w_t + \alpha^\w_t (\b_t - \A_t \w_t), \qq{where} \\
\A_t & \textstyle{\doteq f_{\thet}(S_t)\boldsymbol{\phi}(S_t) \sum_{k=t}^{t+n-1} \left(\prod_{i=t}^{k-1}\gamma_{i+1}\rho_i\right)\rho_k \left[\boldsymbol{\phi}(S_k) - \gamma_{k+1} \boldsymbol{\phi}(S_{k+1})\right]^\top} \\
\b_t & \textstyle{\doteq f_{\thet}(S_t)\sum_{k=t}^{t+n-1} \left(\prod_{i=t}^{k-1}\gamma_{i+1}\rho_i\right)\rho_k R_{k+1} \boldsymbol{\phi}(S_t)}
\end{align}

As we use $f_\theta(S_t)$ to reweight the update,
it is convenient to define a diagonal matrix $\D_\mu^\theta$ with diagonal entries $[\D_\mu^\theta]_{ss} \doteq d_\mu(s) f_\theta(s)$ for any state $s$.
In X-ETD($n$),
we approximate $\lim_{t\to \infty} \E_\mu[F_t | S_t = s]$ with $f_\theta(s)$.
In  this,
we also define a ground-truth diagonal matrix $\D_\mu^f$ such that $[\D_\mu^f]_{ss} \doteq d_\mu(s) \lim_{t\to \infty} \E_\mu[F_t | S_t = s]$, and their difference,
$\D_\mu^\epsilon \doteq \D_\mu^\theta - \D_\mu^f$,
is the ($d_\mu$-weighted) function approximation error matrix of the emphasis approximation.
It can be computed that
\begin{align}
    \textstyle{\A \doteq \lim_{t\to\infty} \E[\A_t] = \bPhi^\top \D_\mu^\theta (\I - (\P_\pi \bGamma)^{n}) \bPhi, \quad
    \b \doteq \lim_{t\to\infty} \E[\b_t] = \bPhi^\top \D_\mu^\theta \r_\pi^n},
\end{align}
where $\r_\pi^n \doteq \sum_{i=0}^{n-1} (\P_\pi \bGamma)^i \r_\pi$ is the $n$-step reward vector with $\r_\pi(s) \doteq \sum_a \pi(a|s) r(s, a)$.

\subsection{Variance}
\label{sec:variance}
Learning to estimate the emphatic trace not only makes value-learning compatible with offline learning methods that make use of replay buffers, but is also instrumental in reducing the variance of the value-learning updates. 
The incremental update of ETD($n$) in~\eqref{eq:etdn} can be rewritten as $F_t \Delta^\w_t$.

The following proposition shows that when the trace approximation error is small enough,
variance in learning can indeed be reduced by replacing the {\notlearned} trace $F_t$ with the learned trace $f_\theta(S_t)$.
\begin{proposition}
\label{prop xnetd reduced var}
(Reduced variance)
Let $\epsilon_s \doteq |f_\theta(s) - f(s)|$ be the trace approximation error at a state $s$. 
For any $s$,
there exists a time $\bar t > 0$, such that for all $t > \bar t$,
\begin{align}
\epsilon_s(\epsilon_s + 2f(s)) < \V(F_t | S_t = s) \, \implies \, \V(f_\theta(S_t) \Delta_t^\w | S_t = s) \leq \V(F_t \Delta_t^\w | S_t = s).
\end{align}
The inequality is strict if $\V(\Delta_t^\w | S_t = s) > 0$.
\end{proposition}
(Proof in the appendix.)
In some cases $\V(F_t | S_t = s)$ can be infinite \citep{sutton2016emphatic}; then the condition in Proposition~\ref{prop xnetd reduced var} holds trivially. This also underpins the importance of variance reduction.

\subsection{Convergence}
\label{sec:convergence}

Next, under the following assumption about the learning rate, 
we show the convergence of \eqref{eq xetdn}.
\begin{assumption}
\label{assu learning rate}
(Learning rates)
    The learning rates $\qty{\alpha_t^\w}_{t=0}^\infty$ are nonnegative, deterministic, and satisfy
    $\sum_t \alpha_t^\w = \infty, \sum_t (\alpha_t^\w)^2 < \infty$.
\end{assumption}
\begin{theorem}
\label{thm convergence x-netd vtrace}
(Convergence of X-ETD($n$))
Under Assumptions~\ref{assu ergodic}-\ref{assu learning rate},
for the iterates $\qty{\w_t}$ generated by \eqref{eq xetdn},
there exists a constant $\eta > 0$ such that
\begin{align}
    \textstyle{\norm{\D_\mu^\epsilon} < \eta \implies \lim_{t \to \infty} \w_t = \A^{-1}\b} \quad a.s..
\end{align}
\end{theorem}
The proof of this theorem is in the appendix, along with a stability guarantee for the {\model} updates.
Theorem~\ref{thm convergence x-netd vtrace} shows that under some mild conditions,
assuming the function approximation error is not too large,
X-ETD($n$) converges to $\w_\infty \doteq \A^{-1}\b.$.

We now study the performance of $\w_\infty$,
i.e.,
the distance between the value prediction by $\w_\infty$ and the true value function $\v_\pi$.
\begin{proposition}
\label{prop fixed point performance}
(Suboptimality of the fixed point)
Under Assumptions~\ref{assu ergodic} \&~\ref{assu full rank},
there exists positive constants $c_1, c_2,$ and $c_3$ such that
\begin{align}
    \textstyle{\norm{\D_\mu^\epsilon} \leq c_1 \implies \norm{\bPhi \w_\infty - \v_\pi} \leq
    c_2 \norm{\D_\mu^\epsilon} + c_3 \norm{\Pi_{\D_\mu^f} \v_\pi - \v_\pi}_{\D_\mu^f}},
\end{align}
\end{proposition}
where $\norm{\Pi_{\D_\mu^f} \v_\pi - \v_\pi}_{\D_\mu^f}$ is the value estimation error of the unbiased fixed point using the {\notlearned} emphasis. We prove this proposition in the appendix. 

\begin{figure}[h!]
\centering
\includegraphics[width=0.9\columnwidth]{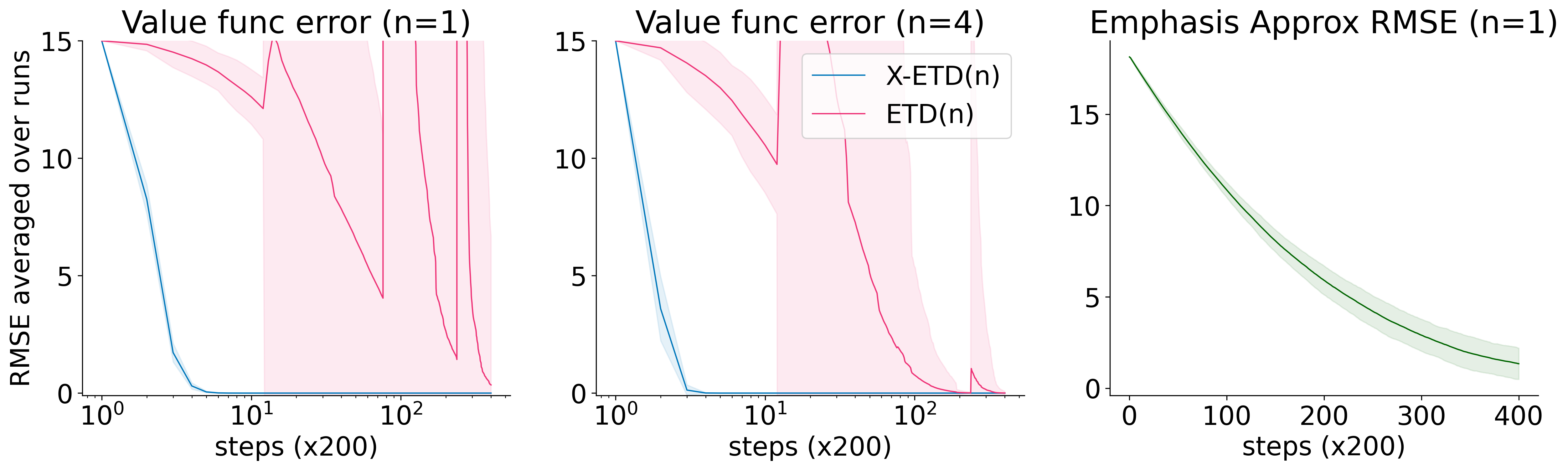}
\caption{RMSE in the value estimates and RMSE in expected trace approximation over time in a modified version of Baird's counterexample \citep{baird1995}. We report the performance of each algorithm using the best performing hyper-parameters (according to RMSE of the value function) from an extensive sweep (described in text). Shaded regions indicate two standard deviations of the mean performance computed from 100 independent runs.}
\vspace{-0.1cm}
\label{fig:baird_exp}
\end{figure}

\subsection{Illustration on Baird's counterexample}
We illustrate the theoretical results in this section on a small MDP modified from Baird's counterexample \citep{baird1995}. The MDP has seven states with linear features. The over-parametrizing features are designed to cause instability even though the true values can be represented. See \citet{sutton2018} for an extensive discussion and analysis of Baird's counterexample. As in \citet{zhang2020provably} we modify the MDP (shown in Fig.~\ref{fig:baird_mdp} in the appendix). The canonical version of Baird's MDP causes all but gradient TD  and residual gradient methods \citep{sutton2009fast} to diverge. Here, we are more interested in a challenging MDP and are not focused on extreme divergence cases. We use $\gamma=0.95$ as discount and a target policy $\pi(\text{solid}|\cdot)=0.3$. We tested all combinations of $\alpha^{\w} \in \{2^i: i=-6,\hdots,-14\}$ and $\alpha^\theta = \alpha^{\w} \beta$, with $\beta \in \{0.0005, 0.001,0.005, 0.01, 0.05,0.1,0.5,1.0, 2.0, 5.0\}$. Figure \ref{fig:baird_exp} summarizes the results with additional results in the appendix.   

{\model} was more stable in this small but challenging MDP. ETD($n$) exhibited high variance and instability. {\model} had very low variance, echoing the conclusion of Prop.~\ref{prop xnetd reduced var} that {\model} has lower variance when its emphasis errors are small. {\model} also converged faster to the true fixed point, 
illustrating Theorem~\ref{thm convergence x-netd vtrace} and, moreover, achieving the optimal fixed point, far better than the worst case upper bound of Prop.~\ref{prop fixed point performance}. Note, even when choosing $\alpha^\w$ and $\beta$ of {\model} to minimize the RMSE in the value function, the emphasis approximation error exhibits steady improvement. 

\section{Experiments}
\label{sec:experiments}
Back to the problem that motivated this paper, our goal is to stabilize learning at scale when using experience replay. Inspired by the performance achieved by Surreal \citep{jiang2021} and previously by StacX \citep{zahavy2020}, both extensions of IMPALA \citep{espeholt2018}, we adopt the same off-policy setting of learning auxiliary tasks to test {\model}, with the additional use of experience replay. It has become conventional to clip IS ratios to reduce variance and improve learning results \citep{espeholt2018, zahavy2020, hessel2021}. We similarly adapt {\model} to the control setting by clipping IS ratios at 1 in both policy evaluation, as described in Clipping of Sec.~\ref{sec:learn_emphasis}, and applying the learned emphatic weighting to the corresponding policy gradients. The setting is similar to \citet{jiang2021}; further details are in the appendix.

\paragraph{Data}
We evaluate {\model} on a widely used deep RL benchmark, Atari games from the Arcade Learning Environment \citep{bellemare2013}\footnote{Licensed under GNU General Public License v2.0.}. The input observations are in RGB format without downsampling or gray scaling. We use an action repeat of 4, with max pooling over the last two frames and the life termination signal. This is the same data format as that used in Surreal \citep{jiang2021} and StacX \citep{zahavy2020}. In addition, we randomly sample half of the training data from an experience replay buffer which contains the most recent 10,000 sequences of length 20. In order to compare with previous works, we use the conventional 200M online frames training scheme, with an evaluation phase at 200M-250M learning frames. 

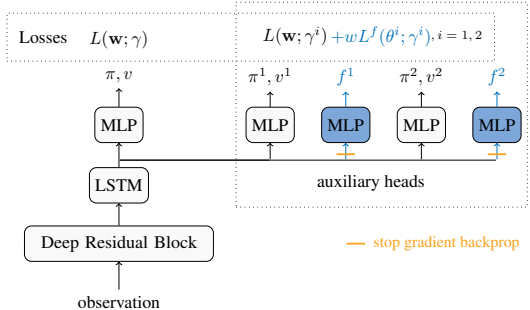
\begin{figure}[t]
\centering
\resizebox{0.5\columnwidth}{!}{
\begin{tikzpicture}
\node (resblock) [mlp] {Deep Residual Block};
\path (resblock.south)+(0, -0.8) node (obs) [inp] {observation};
\path (resblock.north)+(0, 0.8) node (lstm) [op] {LSTM};
\path (resblock.north)+(0, 2) node (mlp0) [op] {MLP};
\path (resblock.north)+(3, 2) node (mlp1) [op] {MLP};
\path (resblock.north)+(6, 2) node (mlp2) [op] {MLP};
\path (resblock.north)+(4.5, 2) node (mlp1_f) [top] {MLP};
\path (resblock.north)+(7.5, 2) node (mlp2_f) [top] {MLP};
\path (mlp0.north)+(0, 0.6) node (out0) [inp] {$\pi, v$};
\path (mlp1.north)+(0, 0.6) node (out1) [inp] {$\pi^1, v^1$};
\path (mlp2.north)+(0, 0.6) node (out2) [inp] {$\pi^2, v^2$};
\path (mlp1_f.north)+(0., 0.6) node (trace1) [tinp] {$f^1$};
\path (mlp2_f.north)+(0., 0.6) node (trace2) [tinp] {$f^2$};
\path (out0.north)+(0, 0.5) node (loss0) [inp] {$L(\w;\gamma)$};
\path (out1.north)+(0.5, 0.5) node (loss1) [inp] {$L(\w;\gamma^i)$};
\path (out1.north)+(2.2, 0.5) node (loss2) [tinp] {$+ w L^f(\theta^i;\gamma^i)$};
\path (out1.north)+(3.7, 0.5) node (loss3) [inp] {\scriptsize $, i=1, 2$};
\node[input] (l) at (-1.5, 4.1) {Losses};
\node[system,fit=(loss0) (loss1) (loss3) (l)] {};
\path (resblock.north)+(0, 1.3) node (null0) [null] {};
\path (null0)+(3, 0) node (null1) [null] {};
\path (null0)+(6, 0) node (null2) [null] {};
\path (null0)+(4.5, 0) node (null1_f) [null] {};
\path (null0)+(7.5, 0) node (null2_f) [null] {};
\path (null2)+(0, 3.) node (null_p1) [null] {};
\path (null_p1)+(0.85, 0) node (null_p2) [null] {};
\node[input] (aux) at (5, 1.2) {auxiliary heads};
\node[system,fit=(mlp1) (mlp2) (mlp2_f) (out1) (out2) (aux) (null1) (null2) (null_p1) (null_p2)] {};

\path [draw, ->] (obs.north) -- node [above] {} (resblock.south);
\path [draw, ->] (lstm.north) -- node [above] {} (mlp0.south);
\path [draw, ->] (resblock.north) -- node [above] {} (lstm.south);
\path [draw, -] (null0) -- node [above] {} (null1);
\path [draw, -] (null0) -- node [above] {} (null2_f);
\path [draw, ->] (null1) -- node [above] {} (mlp1.south);
\path [draw, ->] (null2) -- node [above] {} (mlp2.south);
\path [draw, ->, RoyalBlue] (null1_f) -- node [above] {} (mlp1_f.south);
\path [draw, ->, RoyalBlue] (null2_f) -- node [above] {} (mlp2_f.south);
\path [draw, ->] (mlp0.north) -- node [above] {} (out0.south);
\path [draw, ->] (mlp1.north) -- node [above] {} (out1.south);
\path [draw, ->] (mlp2.north) -- node [above] {} (out2.south);
\path [draw, ->, RoyalBlue] (mlp1_f.north) -- node [above] {} (trace1.south);
\path [draw, ->, RoyalBlue] (mlp2_f.north) -- node [above] {} (trace2.south);
\path (null1_f.north)+(-0.2, 0.1) node (sg1_left) [null] {};
\path (null1_f.north)+(0.2, 0.1) node (sg1_right) [null] {};
\path [draw, -, YellowOrange, line width=1pt] (sg1_left) -- node [above] {} (sg1_right);
\path (null2_f.north)+(-0.2, 0.1) node (sg2_left) [null] {};
\path (null2_f.north)+(0.2, 0.1) node (sg2_right) [null] {};
\path [draw, -, YellowOrange, line width=1pt] (sg2_left) -- node [above] {} (sg2_right);

\node[remark] (sg_text) at (6.5, 0) {\footnotesize stop gradient backprop};
\path (sg_text)+(-2, 0) node (sg_left) [null] {};
\path (sg_text)+(-1.6, 0) node (sg_right) [null] {};
\path [draw, -, YellowOrange, line width=1pt] (sg_left) -- node [above] {} (sg_right);
\end{tikzpicture}}
\caption{Block diagram of {\model}. Agent has one main head, two auxiliary heads and two trace heads with stop gradients. The subset of architecture in gray denotes the baseline agent Surreal, and those highlighted in blue are used in trace learning. We use the IMPALA loss on each head with different discounts $\gamma, \gamma^1, \gamma^2$. The behavior policy is fixed to be $\pi$. The predicted traces $f^1, f^2$ are learned with time-reversed TD losses $L^f(\theta^i;\gamma^i)$ weighted by $w$ for the auxiliary heads $i=1,2$.}
\label{fig:surreal}
\end{figure}

\paragraph{Baseline Agent} Surreal is an IMPALA-based agent that learns two auxiliary tasks with different discounts $\gamma^1, \gamma^2$ simultaneously while learning the main task (Fig.~\ref{fig:surreal}, in gray). The auxiliary tasks are learned off-policy since the agent generates behaviours only from the main policy output. Prior to applying {\model}, we swept extensively on its hyper-parameters to produce the best baseline we could find with 50\% replay data (see the appendix for further details and hyper-parameters). 

\paragraph{{\model} Agents}
We investigate whether {\model} updates can improve off-policy learning of the auxiliary tasks. For each of the two auxiliary tasks, we implement an additional Multilayer Perceptron (MLP) that predicts the expected emphatic trace using the time-reversed TD learning loss $L^f$ in \eqref{eq:learning_f} (see Fig.~\ref{fig:surreal}, in blue). The prediction outputs $f^i$ are then used to re-weight both the V-trace value and policy updates for the auxiliary task $i=1,2$, similar \citet{jiang2021}. In order to isolate the effect of using {\model} from any changes to internal representations as a result of the additional trace learning losses, we prevent the gradients from back-propagating to the core of the agent. We denote learned emphasis with auxiliary Monte Carlo loss as {\model}-MC, described in Sec.~\ref{sec:learn_emphasis}.
We implement all agents in a distributed system based on JAX libraries \citep{haiku2020, rlax2020, optax2020}\footnote{Licensed under Apache License 2.0.} using a TPU Pod infrastructure called Sebulba \citep{sebulba2021}.

\paragraph{Evaluation}
Running many seeds on all 57 Atari games is expensive. However a single metric with few seeds can be noisy, or hard to properly interpret. Hence we adopt a four-faceted evaluation strategy. We report mean and median human normalized training curves with standard deviations across 3 seeds (Fig.~\ref{fig:training_curves}), accompanied by a bar plot of per-game improvements in normalized scores averaged across 3 seeds and the evaluation window (Fig.~\ref{fig:barplot}). In addition, to test rigorously whether {\model} improved performance, we apply a one-sided \emph{Sign Test} \citep{Arbuthnot1712} on independent pairs of agent scores on 57 (games) x 3 (seeds) to compute its $p$-value, where scores are averaged across the evaluation window and the baseline and test agent seeds are paired randomly. To guarantee random pairing, we uniformly sample and pair the seeds 10,000 times and take the average number of games on which the test agent is better. The $p$-value is the probability of observing the stated results under the null hypothesis that the algorithm performs equally. Results might be thought of as statistically significant when $p<0.05$.

\paragraph{Results}

\begin{table}[h!]
\vspace{-0.3cm}
\begin{center}
\scalebox{0.9}{%
\begin{tabular}{l|ccc}
Statistics & X-ETD($n$) & X-ETD($n$)-MC & baseline: Surreal\\
\toprule
Median human normalized score   &503  &\textbf{537}  &525  \\
Mean human normalized score    &\textbf{2122} &2090  &1879\\
$\#$ games > baseline        &\textbf{97} (out of 171)  &\textbf{97} (out of 171)  &N/A\\
$p$-value from Sign Test    &\textbf{0.046}  &\textbf{0.046} &N/A\\
\bottomrule
\end{tabular}}
\end{center}
\caption{Performance statistics for Surreal and learned emphases applied to Surreal on 57 Atari games. Scores are human normalized, averaged across 3 seeds and across the evaluation phase (200M-250M).}
\label{tab:atari_results}
\vspace{-0.3cm}
\end{table}

\begin{figure}[t]
\centering
\includegraphics[width=0.8\columnwidth]{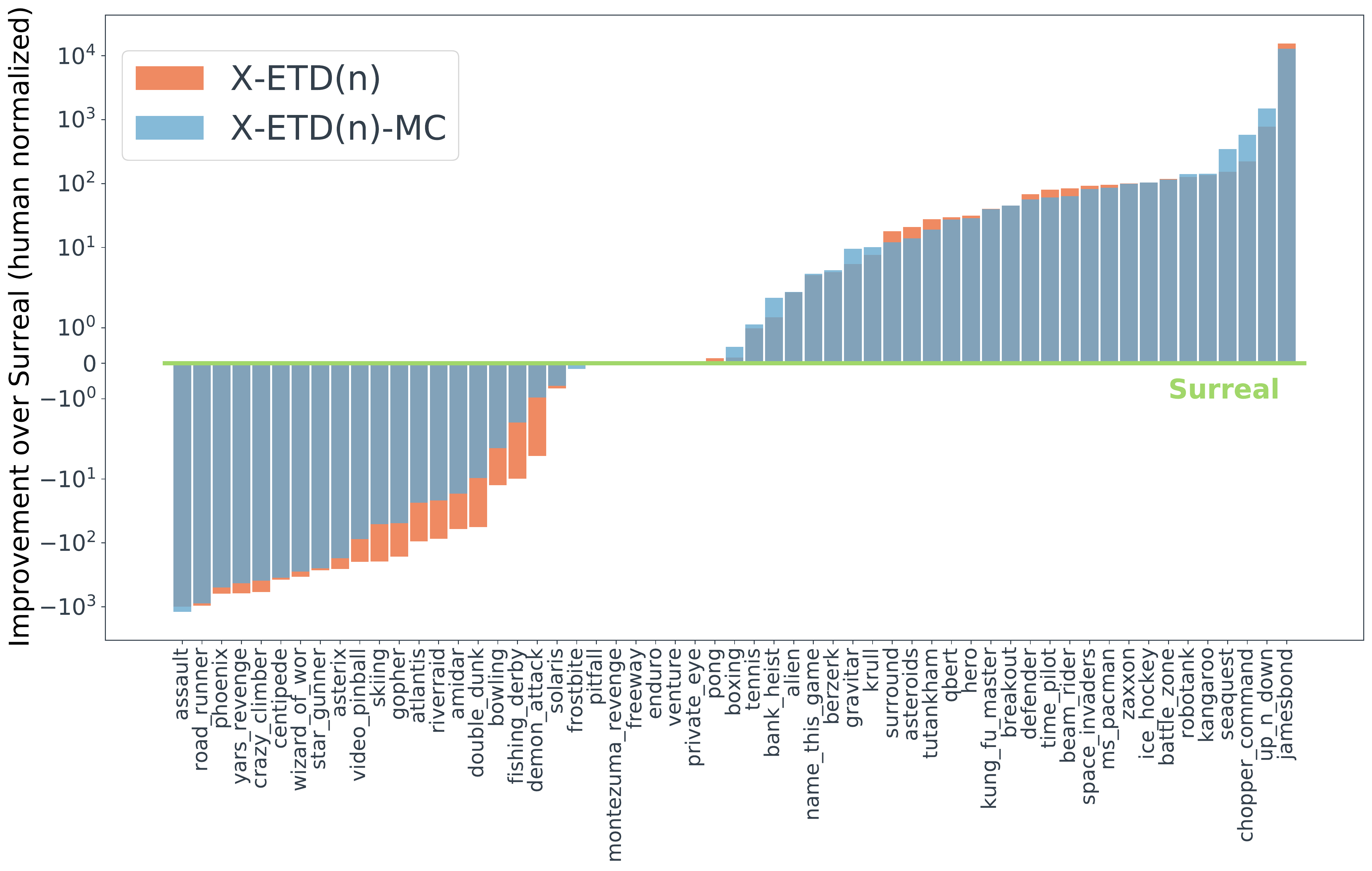}
\caption{Improvement in individual human normalized game scores compared to Surreal: X-ETD($n$) in orange and X-ETD($n$)-MC in blue. Both improve over baseline Surreal in median human normalized scores on 57 Atari games, and scores improved in more games than in which they deteriorated. Results averaged between 200M - 250M frames (evaluation phase) across 3 seeds.}
\vspace{-0.5cm}
\label{fig:barplot}
\end{figure}

\begin{figure}[h!]
\centering
\subfigure[Median human normalized scores]{
\includegraphics[width=0.35\columnwidth, height=2in]{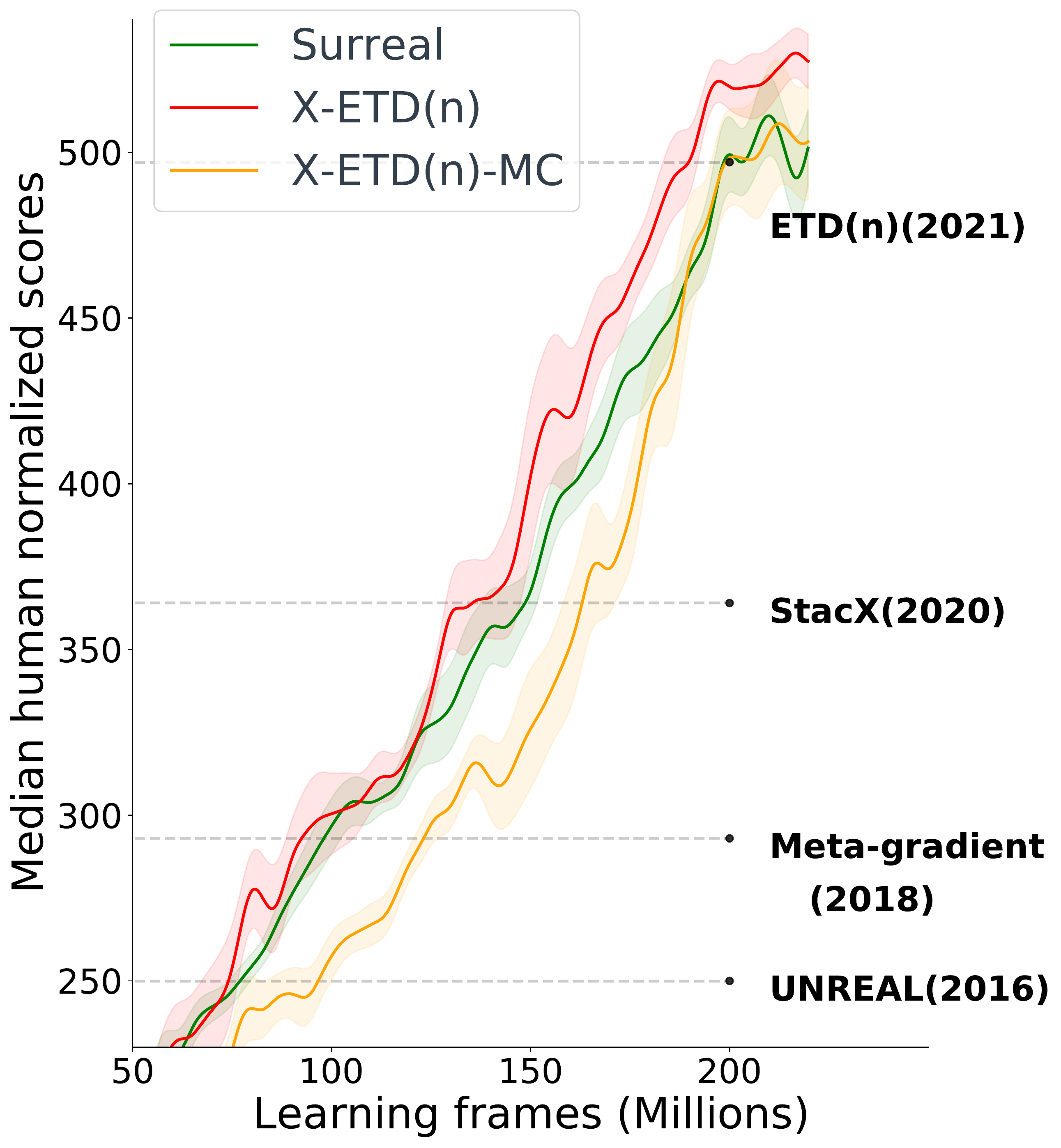}
\label{fig:median_scores}
}
\hspace{1cm}
\subfigure[Mean human normalized scores]{
\includegraphics[width=0.35\columnwidth]{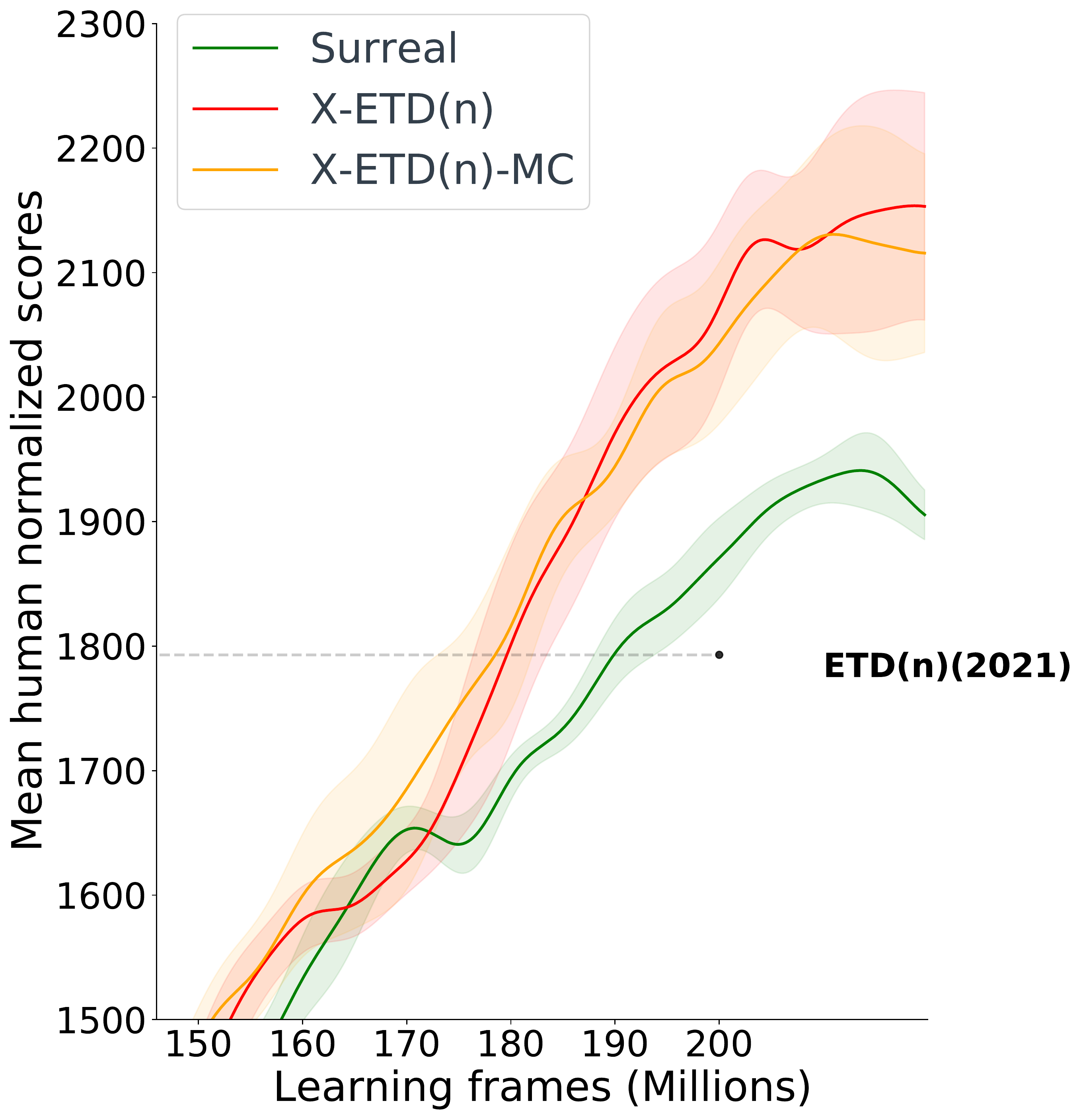}
\label{fig:mean_scores}
}
\caption{Training curves of \subref{fig:median_scores} median and \subref{fig:mean_scores} mean human normalized scores on 57 Atari games, with standard deviations (shaded area) across 3 seeds.}
\label{fig:training_curves}
\end{figure}

Table~\ref{tab:atari_results}, shows that the baseline Surreal agent with 50\% replay achieved a mean score of 1879, a median score of 525; X-ETD($n$) using the same amount of replay data achieved a mean score of 2122, a median score of 503; while X-ETD($n$)-MC obtained mean 2090 and median 537. The median scores are human-normalized and averaged across 3 seeds, then averaged over 200-250M frames evaluation window. On a per-game level, both X-ETD($n$) and X-ETD($n$)-MC outperformed Surreal on 97 out of 57x3=171 games with a p-value of 0.046. Fig.~\ref{fig:barplot} shows per-game improvements over Surreal. The improvements are stable across two emphatic variants on the majority of games. In the games where learned emphasis hurt performance, we observed two scenarios: 1) the predicted emphasis collapsed to 1, e.g. \texttt{assault}, \texttt{road\_runner}, 2) the predicted emphasis runs wild to huge values, e.g. \texttt{yars\_revenge}, \texttt{crazy\_climber}, \texttt{centipede}. The huge negative predictions are especially detrimental as the gradient directions are flipped.

In Fig.~\ref{fig:training_curves} shows mean and median training curves, with standard deviations across 3 seeds for each learning frame and then smoothed {\it via} a standard 1-D Gaussian filter ($\sigma=10$) for clarity. Adding auxiliary Monte Carlo loss is a double-edged sword, while improving stability of the time-reversed TD learning, it also brings in higher variance from the Monte Carlo emphasis. For this reason X-ETD($n$)-MC exhibited more variance than X-ETD($n$) during training, however, it also lead to more stability in score improvement across games, especially mitigating losses where learned emphasis failed to help (see Fig.~\ref{fig:barplot}).

What {\model} approximates is essentially similar to the density ratio between the state distributions of the target and behavior policies in that they both share a backward Bellman equation \citep{liu2018breaking,zhang2020provably}.
Learning density ratios is an active research area but
past works usually only tested on benchmarks with low-dimensional observations (e.g. MuJoCo \citep{todorov2012mujoco}, \citep{liu2018breaking,nachum2019dualdice,zhang2020gradientdice,uehara2020minimax,yang2020off}).
In this work,
we demonstrate that performance improvement is entirely possible when applying learned emphasis to challenging Atari games using high-dimensional image observations.

\section{Related Work}

The idea of learning expected emphatic traces as a function of state has been explored before on the canonical followon trace for backward view TD($\lambda$), to improve trackability of the critic in off-policy actor-critic algorithms \citep{zhang2020provably}. However in this work we focus on the $n$-step trace from \citet{jiang2021} in the forward view, to improve data efficiency in deep RL. Our proposed stabilization techniques, to facilitate at-scale learning, differ from \citet{zhang2020provably}. Though \citet{zhang2020provably} also use a learned trace to reweight 1-step off-policy TD in the GEM-ETD algorithm, theoretical analyses were not provided. In contrast, we provide a thorough theoretical analysis for X-ETD($n$). Finally, we demonstrate the effectiveness of our methods in challenging Atari domains, while \citet{zhang2020provably} experiment with only small diagnostic environments.

The idea of bootstrapping in the reverse direction has also been explored by 
\citet{wang2007dual,wang2008stable,hallak2017consistent,gelada2019off} in learning density ratios 
and by \citet{zhang2020learning} in learning reverse general value functions to represent retrospective knowledge. 
Besides learning a \emph{scalar} followon trace, 
\citet{van2020expected} learn a \emph{vector} eligibility trace \citep{Sutton:1988}, 
which,
together with \citet{satija2020constrained},
inspired our use of an auxiliary Monte Carlo loss.

\section{Conclusion}
\label{sec:conclusion}
In this paper, we propose a simple time-reversed TD learning algorithm for learning expected emphases that is applicable to non-sequential i.i.d. data. We proved that under certain conditions the resulting algorithm {\model} has low variance, is stable and convergence to a reasonable fixed point. Furthermore, it improved off-policy learning results upon well-established baselines on Atari 2600 games, demonstrating its generality and wide applicability. In future works, we would like to study {\model} in more diverse off-policy learning settings using different data sources.

\newpage
\bibliography{paper}

\begin{thebibliography}{}

\bibitem[Arbuthnot, 1712]{Arbuthnot1712}
Arbuthnot, J. (1712).
\newblock {II.} {An} argument for divine providence, taken from the constant
  regularity observ'd in the births of both sexes. {By Dr. John Arbuthnott},
  {Physitian in Ordinary to Her Majesty, and Fellow of the College of
  Physitians and the Royal Society}.
\newblock {\em Philosophical Transactions of the Royal Society of London},
  27(328):186--190.

\bibitem[Baird, 1995]{baird1995}
Baird, L. (1995).
\newblock Residual algorithms: Reinforcement learning with function
  approximation.
\newblock {\em Proceedings of the Twelfth International Conference on Machine
  Learning}, pages 30--37.

\bibitem[Bellemare et~al., 2013]{bellemare2013}
Bellemare, M.~G., Naddaf, Y., Veness, J., and Bowling, M. (2013).
\newblock The arcade learning environment: An evaluation platform for general
  agents.
\newblock {\em Journal of Artificial Intelligence Research}, 47:253–279.

\bibitem[Bertsekas and Tsitsiklis, 1995]{bertsekas1995neuro}
Bertsekas, D.~P. and Tsitsiklis, J.~N. (1995).
\newblock Neuro-dynamic programming: an overview.
\newblock In {\em Proceedings of 1995 34th IEEE conference on decision and
  control}, volume~1, pages 560--564. IEEE.

\bibitem[Budden et~al., 2020]{rlax2020}
Budden, D., Hessel, M., Quan, J., Kapturowski, S., Baumli, K., Bhupatiraju, S.,
  Guy, A., and King, M. (2020).
\newblock {RL}ax: {R}einforcement {L}earning in {JAX}.

\bibitem[Espeholt et~al., 2018]{espeholt2018}
Espeholt, L., Soyer, H., Munos, R., Simonyan, K., Mnih, V., Ward, T., Doron,
  Y., Firoiu, V., Harley, T., Dunning, I., Legg, S., and Kavukcuoglu, K.
  (2018).
\newblock {IMPALA:} scalable distributed deep-rl with importance weighted
  actor-learner architectures.
\newblock {\em CoRR}.

\bibitem[Gelada and Bellemare, 2019]{gelada2019off}
Gelada, C. and Bellemare, M.~G. (2019).
\newblock Off-policy deep reinforcement learning by bootstrapping the covariate
  shift.
\newblock In {\em Proceedings of the AAAI Conference on Artificial
  Intelligence}, volume~33, pages 3647--3655.

\bibitem[Ghiassian et~al., 2018]{ghiassian2018online}
Ghiassian, S., Patterson, A., White, M., Sutton, R.~S., and White, A. (2018).
\newblock Online off-policy prediction.
\newblock {\em CoRR}, abs/1811.02597.

\bibitem[Golub and Van~Loan, 2013]{golub2013matrix}
Golub, G.~H. and Van~Loan, C.~F. (2013).
\newblock {\em Matrix computations}, volume~3.
\newblock JHU press.

\bibitem[Hallak and Mannor, 2017]{hallak2017consistent}
Hallak, A. and Mannor, S. (2017).
\newblock Consistent on-line off-policy evaluation.
\newblock In {\em International Conference on Machine Learning}, pages
  1372--1383. PMLR.

\bibitem[Hennigan et~al., 2020]{haiku2020}
Hennigan, T., Cai, T., Norman, T., and Babuschkin, I. (2020).
\newblock {H}aiku: {S}onnet for {JAX}.

\bibitem[Hessel et~al., 2020]{optax2020}
Hessel, M., Budden, D., Viola, F., Rosca, M., Sezener, E., and Hennigan, T.
  (2020).
\newblock Optax: composable gradient transformation and optimisation, in {JAX}!

\bibitem[Hessel et~al., 2021a]{hessel2021}
Hessel, M., Danihelka, I., Viola, F., Guez, A., Schmitt, S., Sifre, L., Weber,
  T., Silver, D., and van Hasselt, H. (2021a).
\newblock Muesli: Combining improvements in policy optimization.
\newblock In {\em International Conference on Machine Learning}. PMLR.

\bibitem[Hessel et~al., 2021b]{sebulba2021}
Hessel, M., Kroiss, M., Clark, A., Kemaev, I., Quan, J., Keck, T., Viola, F.,
  and van Hasselt, H. (2021b).
\newblock Podracer architectures for scalable reinforcement learning.

\bibitem[Hessel et~al., 2018]{rainbow}
Hessel, M., Modayil, J., van Hasselt, H., Schaul, T., Ostrovski, G., Dabney,
  W., Horgan, D., Piot, B., Azar, M., and Silver, D. (2018).
\newblock Rainbow: Combining improvements in deep reinforcement learning.
\newblock {\em Proceedings of the AAAI Conference on Artificial Intelligence},
  32(1).

\bibitem[Hessel et~al., 2019]{hessel2019multi}
Hessel, M., Soyer, H., Espeholt, L., Czarnecki, W., Schmitt, S., and {van
  Hasselt}, H. (2019).
\newblock Multi-task deep reinforcement learning with popart.
\newblock {\em Proceedings of the AAAI Conference on Artificial Intelligence},
  33(01):3796--3803.

\bibitem[Imani et~al., 2018]{imani2018}
Imani, E., Graves, E., and White, M. (2018).
\newblock An off-policy policy gradient theorem using emphatic weightings.
\newblock {\em Proceedings of the 32nd International Conference on Neural
  Information Processing Systems (NeurIPS 2018)}.

\bibitem[Jiang et~al., 2021]{jiang2021}
Jiang, R., Zahavy, T., White, A., Xu, Z., Hessel, M., Blundell, C., and {van
  Hasselt}, H. (2021).
\newblock Emphatic algorithms for deep reinforcement learning.
\newblock In {\em International Conference on Machine Learning}. PMLR.

\bibitem[Levin and Peres, 2017]{levin2017markov}
Levin, D.~A. and Peres, Y. (2017).
\newblock {\em Markov chains and mixing times}, volume 107.
\newblock American Mathematical Soc.

\bibitem[Lin, 1992]{lin1992self}
Lin, L.-J. (1992).
\newblock Self-improving reactive agents based on reinforcement learning,
  planning and teaching.
\newblock {\em Machine learning}, 8(3-4):293--321.

\bibitem[Liu et~al., 2018]{liu2018breaking}
Liu, Q., Li, L., Tang, Z., and Zhou, D. (2018).
\newblock Breaking the curse of horizon: Infinite-horizon off-policy
  estimation.
\newblock {\em arXiv preprint arXiv:1810.12429}.

\bibitem[Mahmood et~al., 2017]{Mahmood:2017AB}
Mahmood, A.~R., Yu, H., and Sutton, R.~S. (2017).
\newblock Multi-step off-policy learning without importance sampling ratios.
\newblock {\em arXiv preprint arXiv:1702.03006}.

\bibitem[Mnih et~al., 2015]{mnih2015}
Mnih, V., Kavukcuoglu, K., Silver, D., Rusu, A.~A., Veness, J., Bellemare,
  M.~G., Graves, A., Riedmiller, M., Fidjeland, A.~K., Ostrovski, G., Petersen,
  S., Beattie, C., Sadik, A., Antonoglou, I., King, H., Kumaran, D., Wierstra,
  D., Legg, S., and Hassabis, D. (2015).
\newblock Human-level control through deep reinforcement learning.
\newblock {\em Nature}.

\bibitem[Nachum et~al., 2019]{nachum2019dualdice}
Nachum, O., Chow, Y., Dai, B., and Li, L. (2019).
\newblock Dualdice: Behavior-agnostic estimation of discounted stationary
  distribution corrections.
\newblock {\em arXiv preprint arXiv:1906.04733}.

\bibitem[Satija et~al., 2020]{satija2020constrained}
Satija, H., Amortila, P., and Pineau, J. (2020).
\newblock Constrained markov decision processes via backward value functions.
\newblock In {\em International Conference on Machine Learning}, pages
  8502--8511. PMLR.

\bibitem[Schaul et~al., 2016]{SchaulQAS15}
Schaul, T., Quan, J., Antonoglou, I., and Silver, D. (2016).
\newblock Prioritized experience replay.
\newblock In Bengio, Y. and LeCun, Y., editors, {\em 4th International
  Conference on Learning Representations, {ICLR} 2016, San Juan, Puerto Rico,
  May 2-4, 2016, Conference Track Proceedings}.

\bibitem[Sutton, 1988]{Sutton:1988}
Sutton, R.~S. (1988).
\newblock Learning to predict by the methods of temporal differences.
\newblock {\em Machine learning}, 3(1):9--44.

\bibitem[Sutton and Barto, 2018]{sutton2018}
Sutton, R.~S. and Barto, A.~G. (2018).
\newblock {\em Reinforcement Learning: An Introduction}.
\newblock The MIT Press, Cambridge, MA.

\bibitem[Sutton et~al., 2009]{sutton2009fast}
Sutton, R.~S., Maei, H.~R., Precup, D., Bhatnagar, S., Silver, D.,
  Szepesv{\'a}ri, C., and Wiewiora, E. (2009).
\newblock Fast gradient-descent methods for temporal-difference learning with
  linear function approximation.
\newblock In {\em Proceedings of the 26th Annual International Conference on
  Machine Learning}, pages 993--1000.

\bibitem[Sutton et~al., 2016]{sutton2016emphatic}
Sutton, R.~S., Mahmood, A.~R., and White, M. (2016).
\newblock An emphatic approach to the problem of off-policy temporal-difference
  learning.
\newblock {\em The Journal of Machine Learning Research}, 17(1):2603--2631.

\bibitem[Todorov et~al., 2012]{todorov2012mujoco}
Todorov, E., Erez, T., and Tassa, Y. (2012).
\newblock Mujoco: A physics engine for model-based control.
\newblock In {\em 2012 IEEE/RSJ International Conference on Intelligent Robots
  and Systems}, pages 5026--5033. IEEE.

\bibitem[Uehara et~al., 2020]{uehara2020minimax}
Uehara, M., Huang, J., and Jiang, N. (2020).
\newblock Minimax weight and q-function learning for off-policy evaluation.
\newblock In {\em International Conference on Machine Learning}, pages
  9659--9668. PMLR.

\bibitem[van Hasselt et~al., 2018]{hasselt2018}
van Hasselt, H., Doron, Y., Strub, F., Hessel, M., Sonnerat, N., and Modayil,
  J. (2018).
\newblock Deep reinforcement learning and the deadly triad.
\newblock {\em CoRR}, abs/1812.02648.

\bibitem[van Hasselt et~al., 2019]{Hasselt2019WhenTU}
van Hasselt, H., Hessel, M., and Aslanides, J. (2019).
\newblock When to use parametric models in reinforcement learning?
\newblock In {\em Advances in Neural Information Processing Systems 36,
  NeurIPS}.

\bibitem[{van Hasselt} et~al., 2020]{van2020expected}
{van Hasselt}, H., Madjiheurem, S., Hessel, M., Silver, D., Barreto, A., and
  Borsa, D. (2020).
\newblock Expected eligibility traces.
\newblock {\em arXiv preprint arXiv:2007.01839}.

\bibitem[Varga, 1962]{varga1962iterative}
Varga, R.~S. (1962).
\newblock {\em Iterative analysis}.
\newblock Springer.

\bibitem[Wang et~al., 2007]{wang2007dual}
Wang, T., Bowling, M., and Schuurmans, D. (2007).
\newblock Dual representations for dynamic programming and reinforcement
  learning.
\newblock In {\em 2007 IEEE International Symposium on Approximate Dynamic
  Programming and Reinforcement Learning}, pages 44--51. IEEE.

\bibitem[Wang et~al., 2008]{wang2008stable}
Wang, T., Bowling, M., Schuurmans, D., and Lizotte, D.~J. (2008).
\newblock Stable dual dynamic programming.
\newblock In {\em Advances in neural information processing systems}, pages
  1569--1576.

\bibitem[Watkins and Dayan, 2004]{Watkins2004Qlearning}
Watkins, C. J. C.~H. and Dayan, P. (2004).
\newblock Q-learning.
\newblock {\em Machine Learning}, 8:279--292.

\bibitem[White, 2017]{white2017unifying}
White, M. (2017).
\newblock Unifying task specification in reinforcement learning.
\newblock In {\em International Conference on Machine Learning}, pages
  3742--3750. PMLR.

\bibitem[Yang et~al., 2020]{yang2020off}
Yang, M., Nachum, O., Dai, B., Li, L., and Schuurmans, D. (2020).
\newblock Off-policy evaluation via the regularized lagrangian.
\newblock {\em arXiv preprint arXiv:2007.03438}.

\bibitem[{Zahavy} et~al., 2020]{zahavy2020}
{Zahavy}, T., {Xu}, Z., {Veeriah}, V., {Hessel}, M., {Oh}, J., {van Hasselt},
  H., {Silver}, D., and {Singh}, S. (2020).
\newblock A self-tuning actor-critic algorithm.
\newblock {\em 34th Conference on Neural Information Processing Systems
  (NeurIPS 2020)}.

\bibitem[Zhang et~al., 2019]{zhang2019}
Zhang, S., Boehmer, W., and Whiteson, S. (2019).
\newblock Generalized off-policy actor-critic.
\newblock In {\em Advances in Neural Information Processing Systems},
  volume~32.

\bibitem[Zhang et~al., 2020a]{zhang2020gradientdice}
Zhang, S., Liu, B., and Whiteson, S. (2020a).
\newblock Gradientdice: Rethinking generalized offline estimation of stationary
  values.
\newblock In {\em International Conference on Machine Learning}, pages
  11194--11203. PMLR.

\bibitem[Zhang et~al., 2020b]{zhang2020provably}
Zhang, S., Liu, B., Yao, H., and Whiteson, S. (2020b).
\newblock Provably convergent two-timescale off-policy actor-critic with
  function approximation.
\newblock In {\em International Conference on Machine Learning}, pages
  11204--11213. PMLR.

\bibitem[Zhang et~al., 2020c]{zhang2020learning}
Zhang, S., Veeriah, V., and Whiteson, S. (2020c).
\newblock Learning retrospective knowledge with reverse reinforcement learning.
\newblock In {\em Advances in Neural Information Processing Systems},
  volume~33.

\end{thebibliography}
\bibliographystyle{apalike}

\appendix
\newpage
\vbox{%
\hsize\textwidth
\linewidth\hsize
\vskip 0.1in
\hrule height 4pt
\vskip 0.25in
\vskip -\parskip%
\centering
{\LARGE\bf Supplementary Material\par}
\vskip 0.29in
\vskip -\parskip
\hrule height 1pt
\vskip 0.09in%
\vskip 0.3in}

\title{Supplementary Material}
\section{Time-reversed TD}
\paragraph{Instability} 

The asymptotic update matrix of~\eqref{eq:backward_td} is
\begin{align}
\A &\doteq \lim_{k\rightarrow\infty} \E[\A_k] = \lim_{k\rightarrow\infty}  \E \left[\phi(S^k_n)\left[\phi(S^k_n)-\left(\prod_{t=1}^{n} \gamma^k_t\rho^k_{t-1}\right) \phi(S^k_0) \right]^\top \right] \\
&= \sum_s d_{\mu}(s) \E \left[\phi(S^k_n\left[\phi(S^k_n)-\left(\prod_{t=1}^{n} \gamma^k_t\rho^k_{t-1}\right) \phi(S^k_0) \right]^\top \mid S^k_n=s \right]\\
&= \bPhi^\top \D_{\mu} (\I- \D_{\mu}^{-1} (\bGamma \P^T_{\pi})^{n} \D_{\mu}) \bPhi,\\
&= \bPhi^T (\I- (\bGamma \P^T_{\pi})^{n}) \D_{\mu} \bPhi.
\end{align}
Notice that $\A$ is not necessarily p.d.. It is the matrix transpose of the steady state $n$-steps TD update matrix $\bPhi^\top \D_{\mu} (\I- (\P_{\pi}\bGamma)^{n}) \bPhi$. Thus the time-reversed TD is not always stable.

We now state a Lemma rephrased from \citet{sutton2016emphatic},
which will be repeatedly used in this paper.
\begin{lemma}
\label{lem sutton}
\citep{sutton2016emphatic}
Let $\X$ be a matrix with full column rank, $\D$ be a diagonal matrix with strictly positive diagonal entries,
$\P$ be a substochastic matrix, the row vector $\textbf{1}^\top \D(\I - \P)$ be elementwise strictly positive, 
then 
$\X^\top \D(\I - \P) \X$ is p.d..
\end{lemma}
\begin{proof}
We first show that
\begin{align}
    \Y \doteq \D(\I - \P) + (\D(\I - \P))^\top
\end{align}
is p.d..
Since $\Y$ is symmetric,
Corollary in page 23 of \citet{varga1962iterative} states that $\Y$ is p.d. if $\Y$ is strcitly diagonally dominant,
i.e.,
if for any $i$,
\begin{align}
\label{eq diagonally dominant}
    |\Y(i, i)| > \sum_{j \neq i} |\Y(i, j)|.
\end{align}
Note that the diagonal entries of $\Y$ are nonnegative and the off-diagnoal entries of $\Y$ are nonpositive.
Consequently, \eqref{eq diagonally dominant} is equivalent to $(\Y \textbf{1})(i) > 0$. 
We have
\begin{align}
    \Y \textbf{1} = \D(\I - \P) \textbf{1} + (\textbf{1}^\top \D(\I - \P))^\top.
\end{align}
Since $\P$ is a substochastic matrix,
$((\I - \P) \textbf{1})(i) \geq 0$ holds for any $i$,
it is then easy to see $\Y \textbf{1}(i) > 0$.
Consequently, $\Y$ is p.d..
By \citet{Sutton:1988},
$\D(\I - \P)$ is p.d. as well,
implying that $\X \D(\I - \P) \X$ is p.d..
\end{proof}

\paragraph{Proof of Proposition~\ref{prop trace clip}}
\begin{proof}
It suffices,
according to Lemma~\ref{lem sutton},
to show that the row vector 
\begin{align}
    \k^\top \doteq \d_\mu^\top - \d_\mu^\top (\P_{\bar \rho})^{n}
\end{align}
is strictly element-wise positive.
Let
\begin{align}
    u_{\bar \rho} &\doteq \max_{s, a} \min(\rho(a | s), \bar \rho), \\
    \gamma &\doteq \max_s \gamma(s).
\end{align}
we have
\begin{align}
   &\P_{\bar \rho}(s, s') \leq \sum_a \mu(a | s) p(s' | s, a) u_{\bar \rho} \gamma(s') = \P_\mu(s, s') u_{\bar \rho} \gamma(s') \leq \P_\mu(s, s') u_{\bar \rho} \gamma,\\
\implies &\sum_j \P_{\bar \rho}(s, j) \P_{\bar \rho}(j, s') \leq \sum_j \P_\mu(s, j) \P_\mu(j, s') u_{\bar \rho}^2 \gamma^2\\
\implies &\P_{\bar \rho}^2(s, s') \leq \P_\mu^2(s, s') u_{\bar \rho}^2 \gamma^2\\
\implies &\P_{\bar \rho}^{n}(s, s') \leq \P_\mu^{n}(s, s') u_{\bar \rho}^{n} \gamma^{n}\\
\end{align}
So
\begin{align}
    k(s) &= d_\mu(s) - \sum_{\bar s} d_\mu(\bar s) \P_{\bar \rho}^{n}(\bar s, s) \\
    & \geq d_\mu(s) - \sum_{\bar s} d_\mu(\bar s) \P_\mu^{n}(\bar s, s) u_{\bar \rho}^{n} \gamma^{n}\\
    &= d_\mu(s) (1 - u_{\bar \rho}^{n}\gamma^{n})
\end{align}
 Thus
\begin{align}
    u_{\bar \rho} < \frac{1}{\gamma}
\end{align}
is a sufficient condition for the key matrix to be p.d.,
which completes the proof.
\end{proof}

\paragraph{Proof of Proposition~\ref{prop trace monte carlo}}
\begin{proof}
The asymptotic update matrix of~\eqref{eq sequential_iid_update} is
\begin{align}
\A \doteq \bPhi^\top ((\I+\beta) - (\bGamma \P^\top_{\pi})^{n}) \D_\mu\bPhi.
\end{align}
For this matrix to be p.d.,
it suffices, according to Lemma~\ref{lem sutton},
to have the row vector
\begin{align}
    \mathbf{1}^\top \D_\mu ((1 + \beta) \I - (\P_\pi \bGamma)^{n})
\end{align}
to be strictly element-wise positive.
Clearly,
one sufficient condition is that
\begin{align}
    1 + \beta > \max_s \frac{\left(\d_\mu^\top (\P_\pi \bGamma)^{n}\right)(s)}{d_\mu(s)},
\end{align}
which completes the proof.
\end{proof}

\section{X-ETD($n$)}

\paragraph{Proof of Proposition~\ref{prop xnetd reduced var}}
\begin{proof}
It is easy to see
\begin{align}
    \V(f_\theta(S_t) \Delta_t^\w | S_t = s) = f_\theta^2(s) \V(\Delta_t^\w | S_t = s).
\end{align}
By the rule of the variance of the product of (conditionally) independent random variables,
we have
\begin{align}
    &\V(F_t \Delta_t^\w | S_t = s) \\ =& \V(F_t | S_t = s) \V(\Delta_t^\w | S_t = s) + \V(F_t | S_t = s) \E^2[\Delta_t^\w | S_t = s] + \V(\Delta_t^\w | S_t = s) \E^2[F_t | S_t = s] \\
    \geq& \V(F_t | S_t = s) \V(\Delta_t^\w | S_t = s) + \V(\Delta_t^\w | S_t = s) \E^2[F_t | S_t = s].
\end{align}
Then it is easy to see that one sufficient condition for 
\begin{align}
\V(f_\theta(S_t) \Delta_t^\w | S_t = s) \leq \V(F_t \Delta_t^\w | S_t = s)
\end{align}
to hold is that
\begin{align}
    f_\theta^2(s) \leq \V(F_t | S_t = s) + \E^2[F_t | S_t = s].
\end{align}
For any state $s$,
simple algebraic manipulation shows that
\begin{align}
f_\theta^2(s) - f^2(s) \leq \epsilon_s(\epsilon_s + 2f(s)),
\end{align}
i.e., for any $s$ and $t$,
\begin{align}
&\epsilon_s(\epsilon_s + 2f(s)) < \V(F_t | S_t = s) \\
\implies & f_\theta^2(s) - f^2(s) < \V(F_t | S_t = s).
\end{align}
Or equivalently,
\begin{align}
&\epsilon_s(\epsilon_s + 2f(s)) < \V(F_t | S_t = s) \\
\implies &f_\theta^2(s) = \V(F_t | S_t = s) + f^2(s) - \tau
\end{align}
for some $\tau > 0$.
Since 
\begin{align}
    \lim_{t\to \infty} \E^2[F_t | S_t = s] = f^2(s),
\end{align}
for any $s$, there exists a $\bar t$ such that for all $t > \bar t$,
\begin{align}
|\E^2[F_t | S_t = s] - f^2(s)| < \tau.
\end{align}
Consequently,
for any $t > \bar t$,
\begin{align}
&\epsilon_s(\epsilon_s + 2f(s)) < \V(F_t | S_t = s) \\
\implies &f_\theta^2(s) < \V(F_t | S_t = s) + \E^2[F_t | S_t = s],
\end{align}
which completes the proof.
\end{proof}

\subsection{Stability}
\label{sec:stability}
After making the following assumption about the features, 
we show that the update of X-ETD($n$) \eqref{eq xetdn} is stable as long as the function approximation error is not too large. 
\begin{assumption}
\label{assu full rank}
(Features)
The feature matrix $\bPhi$ has full column rank.
\end{assumption}
\begin{lemma}
\label{lem pd}
(Stability)
Under Assumptions~\ref{assu ergodic} \&~\ref{assu full rank},
there exists a constant $\eta_0 > 0$ such that
\begin{align}
    \norm{\D_\mu^\epsilon} < \eta_0 \implies \A \, \text{is p.d.}
\end{align}
\end{lemma}

\begin{proof}
As shown by \citet{jiang2021},
$\D^{f}_{\mu} (\I - (\P_\pi \bGamma)^{n})$ is p.d.,
i.e.,
for any $\y$, $$g(\y) \doteq \y^\top \D^{f}_{\mu} (\I - (\P_\pi \bGamma)^{n}) \y > 0.$$
Since $g(\y)$ is a continuous function, 
it obtains its minimum value in the compact set $\mathcal{Y} \doteq \qty{\y: \norm{\y} = 1}$,
say, e.g., $\eta$,
i.e.,
\begin{align}
    g(\y) \geq \eta > 0
\end{align}
holds for any $\y \in \mathcal{Y}$.
In particular, for any $\y \in \R^K$,
\begin{align}
    g(\frac{\y}{\norm{\y}}) \geq \eta
\end{align}
i.e.,
\begin{align}
    \y^\top \D^{f}_{\mu} (\I - (\P_\pi \bGamma)^{n}) \y \geq \eta \norm{\y}^2.
\end{align}
Let
\begin{align}
    \eta_0 \doteq \frac{\eta}{\norm{\I - (\P_\pi \bGamma)^{n}}},
\end{align}
we have for any $\y$,
\begin{align}
    &\y^\top \D^{\theta}_{\mu} (\I - (\P_\pi \bGamma)^{n} ) \y \\
    =& \y^\top \D^{f}_{\mu} (\I - (\P_\pi \bGamma)^{n}) \y + \y^\top \D^{\varepsilon}_\mu (\I - (\P_\pi \bGamma)^{n}) \y \\
    \geq & \eta \norm{\y}^2 + \y^\top \D^{\varepsilon}_\mu (\I - (\P_\pi \bGamma)^{n}) \y  \\
    \geq & \eta \norm{\y}^2 - |\y^\top \D^{\varepsilon}_\mu (\I - (\P_\pi \bGamma)^{n}) \y| \\
    \geq & \eta \norm{\y}^2 - \norm{\y}^2 \norm{\D^{\varepsilon}_\mu}  \norm{\I - (\P_\pi \bGamma)^{n}} \\
    =& (\eta_0 - \norm{\D^{\varepsilon}_\mu})  \norm{\I - (\P_\pi \bGamma)^{n}} \norm{\y}^2,
\end{align}
i.e.,
when $\norm{\D_\mu^\epsilon} < \eta_0$ holds,
$\D^{f}_{\theta} (\I - (\P_\pi \bGamma)^{n} )$ is p.d.,
which, together with Assumption~\ref{assu full rank}, immediately implies that $\A$ is p.d..
\end{proof}
Following the definition of stability in \citet{sutton2016emphatic}, $\A$ being p.d. gives stable steady state updates.

\paragraph{Proof of Theorem~\ref{thm convergence x-netd vtrace}}
\begin{proof}
We first consider Eq.~\eqref{eq xetdn} in the sequential setting.

Let $Z_t \doteq (S_t, A_t, \dots, S_{t+n}, A_{t+n}, S_{t+n+1})$.
Assumption~\ref{assu ergodic} implies that the Markov chain $\qty{Z_t}$ is ergodic.
We use $d_z$ to denote its ergodic distribution.
For $z = (s_1, a_1, \dots, s_n, a_n, s_{n+1})$,
we define matrix-valued and vector-valued functions
\begin{align}
A(z) &\doteq f_\theta(s_1) \phi(s_1) \sum_{k=1}^n \left(\prod_{i=1}^{k-1} \gamma(s_{i+1}) \frac{\pi(a_i|s_i)}{\mu(a_i | s_i)}  \right) \frac{\pi(a_k|s_k)}{\mu(a_k | s_k)} (\phi(s_k) - \gamma(s_{k+1}) \phi(s_{k+1}))^\top, \\ 
b(z) &\doteq f_\theta(s_1) \sum_{k=1}^n \left(\prod_{i=1}^{k-1} \gamma(s_{i+1}) \frac{\pi(a_i|s_i)}{\mu(a_i | s_i)}  \right) \frac{\pi(a_k|s_k)}{\mu(a_k | s_k)} r(s_k, a_k) \phi(s_k),
\end{align}
which allows us to rewrite \eqref{eq xetdn} as
\begin{align}
    \w_{t+1} = \w_t + \alpha^\w_t (b(Z_t) - A(Z_t) \w_t).
\end{align}
For stochastic approximation algorithms like this,
Proposition 4.8 in \citet{bertsekas1995neuro} asserts that $\qty{w_t}$ convergs almost surely if the following five conditions are satisfied:

(a) The learning rates $\qty{\alpha^\w_t}$ are nonnegative, deterministic, and satisfy $\sum_t \alpha_t^w = \infty, \sum_t (\alpha_t^w)^2 < \infty$ \\
(b) $\qty{Z_t}$ is ergodic \\
(c) $\E_{d_z}[A(z)]$ is positive definite \\
(d) $\max_z \norm{A(z)} < \infty, \max_z \norm{b(z)} < \infty$ \\
(e) There exist scalars $c$ and $\tau_0$ with $\tau_0 \in [0, 1)$ such that 
\begin{align}
    \norm{\E[A(Z_t)] - \E_{d_z}[A(z)]} \leq c \tau_0^t, \norm{\E[b(Z_t)] - \E_{d_z}[b(z)]} \leq c \tau_0^t.
\end{align}

In our case, (a) is satisfied by Assumption~\ref{assu learning rate}.
(b) follows from Assumption~\ref{assu ergodic}.
Since $\E_{d_z}[A(z)] = \A$,
$\E_{d_z}[b(z)] = \b$,
(c) follows from Lemma~\ref{lem pd}.
(d) is obvious since we consider a finte state action MDP.
To verify (e), 
consider
\begin{align}
    \E[A(Z_t)] = \sum_z \Pr(Z_t = z) A(z).
\end{align}
So
\begin{align}
    \norm{\E[A(Z_t)] - \E_{d_z}[A(z)]} &= \norm{\sum_z \left(\Pr(Z_t = z) - d_z(z) \right) A(z)} \\
    &\leq \sum_z |\Pr(Z_t = z) - d_z(z)| \norm{A(z)} \\
    &\leq \left(\sum_z |\Pr(Z_t = z) - d_z(z)| \right) \max_z \norm{A(z)}
\end{align}
According to Theorem 4.9 in \citet{levin2017markov},
under Assumption~\ref{assu ergodic},
there exists scalar $c_0$ and $\tau_0$ with $\tau_0 \in [0, 1)$ such that
\begin{align}
    \sum_z |\Pr(Z_t = z) - d_z(z)| \leq c_0 \tau_0^t.
\end{align}
Consequently, 
the first condition in (e) is verified;
the second condition in (e) can be verified in the same way.

With all the five conditions satisfied,
Proposition 4.8 in \citet{bertsekas1995neuro} asserts that
\begin{align}
\lim_{t \to \infty} \w_t = \A^{-1}\b \quad a.s..
\end{align}

The ergodic distribution $d_z$ and probability of $Z_t=z$ remain the same in the i.i.d. setting as long as $S_t \sim d_\mu$. Therefore we reach the same conclusion in the i.i.d. setting.
\end{proof}

\paragraph{Proof of Proposition~\ref{prop fixed point performance}}
\begin{proof}
For the sake of readability,
in this proof,
we define
\begin{align}
    \L \doteq \I - (\P_\pi \bGamma)^{n}
\end{align}
as shorthand.
We first bound the distance between $\w_\infty$ and the unbiased fixed point
\begin{align}
    \w_* \doteq (\bPhi^\top \D_\mu^f \L \bPhi)^{-1} \bPhi^\top \D_\mu^f \r_\pi^n.
\end{align}
We have
\begin{align}
\label{eq fixed point perf part1}
    &\norm{\w_\infty - \w_*} \\
    \leq & \norm{\left( (\bPhi^\top \D_\mu^\theta \L \bPhi)^{-1} -(\bPhi^\top \D_\mu^f \L \bPhi)^{-1} \right)\bPhi^\top \D_\mu^\theta  \r_\pi^n} + \norm{(\bPhi^\top \D_\mu^f \L \bPhi)^{-1} \bPhi^\top (\D_\mu^\theta - \D_\mu^f) \r_\pi^n} \\
    \leq & \norm{(\bPhi^\top \D_\mu^\theta \L \bPhi)^{-1}}\norm{(\bPhi^\top \D_\mu^f \L \bPhi)^{-1}} \norm{\bPhi^\top \D_\mu^\epsilon  \L \bPhi} \norm{\bPhi^\top (\D_\mu^f + \D_\mu^\epsilon) \r_\pi^n} + \norm{(\bPhi^\top \D_\mu^f \L \bPhi)^{-1} \bPhi^\top \D_\mu^\epsilon \r_\pi^n}
    \intertext{\hfill (Using $\norm{\X^{-1} - \Y^{-1}} \leq \norm{\X^{-1}}\norm{\Y^{-1}}\norm{\X - \Y}$)}
\end{align}
We now bound $(\bPhi^\top \D_\mu^\theta \L\bPhi)^{-1}$,
According to Corollary 8.6.2 of \citet{golub2013matrix},
for any two matrices $\X$ and $\Y$,
\begin{align}
|\sigma_{min}(\X + \Y) - \sigma_{min}(\X)| \leq \norm{\Y},
\end{align}
where $\sigma_{min}(\cdot)$ indicates the smallest singular value.
If $\X$ is nonsingular and we select some $c_0 \in (0, \sigma_{min}(\X))$,
we get
\begin{align}
&\norm{\Y} \leq \sigma_{min}(\X) - c_0 \\
\implies &\norm{(\X+\Y)^{-1}} = \frac{1}{\sigma_{min}(\X+\Y)} \leq \frac{1}{\sigma_{min}(\X) - \norm{\Y}} \leq c_0^{-1}.
\end{align}
In our case, 
we consider $\bPhi^\top \D_\mu^f \L \bPhi$ as $\X$ and $\bPhi^\top \D_\mu^\epsilon \L \bPhi$ as $\Y$,
we get
\begin{align}
\label{eq fixed point perf part2}
&\norm{\bPhi^\top \D_\mu^\epsilon \L \bPhi} \leq \sigma_{min}(\bPhi^\top \D_\mu^f \L \bPhi) - c_0 \\
\implies & \norm{(\bPhi^\top \D_\mu^\theta \L \bPhi)^{-1}} \leq c_0^{-1}.
\end{align}
Here the nonsingularity of $\bPhi^\top \D_\mu^f \L \bPhi$ is proved in \citet{jiang2021}.
Combining the spectral radius bounds of $\bPhi^\top \D_\mu^f \L \bPhi$ with \eqref{eq fixed point perf part1} and \eqref{eq fixed point perf part2},
it is easy to see that there exists a constant $c_1 > 0$ such that
\begin{align}
    \norm{\w_\infty - \w_*} \leq c_1 \norm{\D_\mu^\epsilon}.
\end{align}
Consequently,
\begin{align}
   \norm{\bPhi \w_\infty - \v_\pi} \leq \norm{\bPhi \w_\infty - \bPhi \w_*} + \norm{\bPhi \w_* - \v_\pi} \leq \norm{\bPhi}c_1 \norm{\D_\mu^\epsilon} + \norm{\bPhi \w_* - \v_\pi}
\end{align}
According to Lemma 1 and Theorem 1 in \citet{white2017unifying},
there exists a constant $c_2 > 0$ such that
\begin{align}
\norm{\bPhi \w_* - \v_\pi}_{\D_\mu^f} \leq c_2 \norm{\Pi_{\D_\mu^f} \v_\pi - \v_\pi}_{\D_\mu^f}.
\end{align}
Using the equivalence between norms,
we have for some constant $c_3 > 0$,
\begin{align}
\norm{\bPhi \w_* - \v_\pi} \leq c_3 \norm{\bPhi \w_* - \v_\pi}_{\D_\mu^f},
\end{align}
which completes the proof.
In particular, one possible $\eta_1$ is 
\begin{align}
\frac{\sigma_{min}(\bPhi^\top \D_\mu^f \L \bPhi) - c_0}{\norm{\bPhi^\top}\norm{\L \bPhi}}.
\end{align}
\end{proof}

\section{IMPALA}
\paragraph{V-trace} Since the product of IS ratios in off-policy TD($n$) can lead to high variances, \emph{V-trace} \citep{espeholt2018} clips the IS ratios to reduce variance.
V-trace updates $\w$ iteratively as 
\begin{align}
\label{eq:vtrace_orig}
\w_{t+1} \doteq \w_t + \alpha \sum_{k=t}^{t+n-1} \left(\prod_{i=t}^{k-1} \bar c_i \gamma_{i+1} \right) \, \bar \rho_k \delta_k(\w_t) \boldsymbol{\phi}_t,
\end{align}
where $\bar{\rho}_t \doteq \min(\bar{\rho}, \rho_t)$, $\bar{c}_t \doteq \min(\bar{c}, \rho_t)$. In practice, the clipping thresholds $\bar{c}$ and $\bar{\rho}$ are often equal, so that $\bar{c}_t \equiv \bar{\rho}_t$.
Clipping due to $\bar{c}$ is equivalent to bootstrapping more, as discussed by \citet{Mahmood:2017AB}.
It is straightforward to apply both off-policy $n$-step TD and V-trace in the i.i.d. setting.

\paragraph{Control} The V-trace update is most often used in actor-critic systems, such as Impala \citep{espeholt2018}. Consider the current policy $\pi_\vartheta$ parametrized by an additional set of policy parameters $\vartheta$. Following the derivation of policy gradient in \citet{espeholt2018}, we update the \emph{actor}, parameters $\vartheta$ in the following direction:
\begin{align}
\label{eq:vtrace_actor}
    \bar{\rho}_t (R_{t+1} + \gamma_{t+1} v_{t+1} - v_{\w}(S_t)) \nabla_{\vartheta}\log \pi_{\vartheta}(A_t|S_t) \,,
\end{align}
where $v_{t+1}$ is the V-trace target,
\begin{align}
\label{eq:vtrace_target}
   v_{t+1} \doteq v_{\w}(S_{t+1}) + \hskip-0.1cm \sum_{k=t+1}^{t+n} \left(\prod_{i=t+1}^{k-1} \bar{c}_i \gamma_{i+1}\right) \hskip0.1cm \bar{\rho}_k \delta_k(\w),
\end{align}
and the \emph{critic} $v_\w$ is learned using the aforementioned policy evaluation V-trace update. This learning algorithm has been very successful \citep[e.g.,][]{espeholt2018,hessel2019multi} in mild off-policy learning settings.

\citet{jiang2021} extends ETD($n$) trace to the control setting through re-weighting the policy gradient in the learning updates as well as the value gradient by the ETD($n$) trace. 
\section{Hyperparameters}
\label{sec:surreal_hyperparams}

\textbf{Architectures.}

\begin{table}[h!]
\caption{Network architecture}
\begin{center}
\begin{tabular}{|l|l|l|l|}
    \hline
    Parameter & \\
    \hline 
    convolutions in block & (2, 2, 2, 2) \\
    channels & (64, 128, 128, 64) \\
    kernel sizes & (3, 3, 3, 3) \\
    kernel strides & (1, 1, 1, 1)  \\
    pool sizes & (3, 3, 3, 3)  \\
    pool strides & (2, 2, 2, 2)  \\
    frame stacking & 4 \\
    head hiddens & 512  \\
    activation & Relu \\
    trace hiddens & 256 \\
    \hline
\end{tabular}
\end{center}
\label{table:hyperparameters2}
\end{table}

Our DNN architecture is composed of a shared torso, which then splits to different heads. We have a head for the policy, a head for the value function and a head for the emphatic trace (multiplied by the number of auxiliary tasks). The value and policy heads are two-layered MLPs with 512 hidden units, where the output dimension corresponds to 1 for the value function head. For the policy head, we have $|A|$ outputs that correspond to softmax logits. The trace head is a two-layered MLP with 256 hidden units, and the output dimension 1. We use ReLU activations on the outputs of all the layers besides the last layer. For the policy head, we apply a softmax layer and use the entropy of this softmax distribution as a regularizer.  

The \textbf{torso} of the network is composed from residual blocks. In each block there is a convolution layer, with stride, kernel size, channels specified in Table \ref{table:hyperparameters2}, with an optional pooling layer following it. The convolution layer is followed by n - layers of convolutions (specified by blocks), with a skip contention. The output of these layers is of the same size of the input so they can be summed. The block convolutions have kernel size $3,$ stride $1$.  

\textbf{Hyperparameters.}
Table \ref{table:hyperparameters} lists all the hyperparameters used by our agent. Most of the hyperparameters follow the reported parameters from the IMPALA paper. For completeness, we list all of the exact values that we used below.  
\begin{table}[h!]
\caption{Hyperparameters table}
\begin{center}
\begin{tabular}{|l|l|l|l|}
    \hline
    Parameter & Value  \\
    \hline 
    total environment steps & 200e6 \\
    optimizer & RMSPROP \\
    start learning rate & $2 \cdot 10^{-4}$ \\
    end learning rate & 0  \\
    trace weight w & 1\\
    decay & 0.99 \\
    eps & 0.1 \\
    importance sampling clip & 1 \\
    gradient norm clip & $1$\\
    trajectory $n$ & $10$\\
    online batch size & 6 \\
    replay batch size & 6 \\
    replay buffer size & $10^4$\\
    sampling priority &  uniform\\
    discount  $\gamma$ (main) &  $\sigma(4.6)\approx.99$ \\
    discount  $\gamma^1$ ($1^{st}$ auxiliary) &  $\sigma(4.4)\approx.988$ \\
    discount  $\gamma^2$ ($2^{nd}$ auxiliary) &  $\sigma(4.2)\approx.985$ \\
    \hline
\end{tabular}
\end{center}
\label{table:hyperparameters}
\end{table}

\newpage
\section{Additional Results}
\label{sec:add_results}

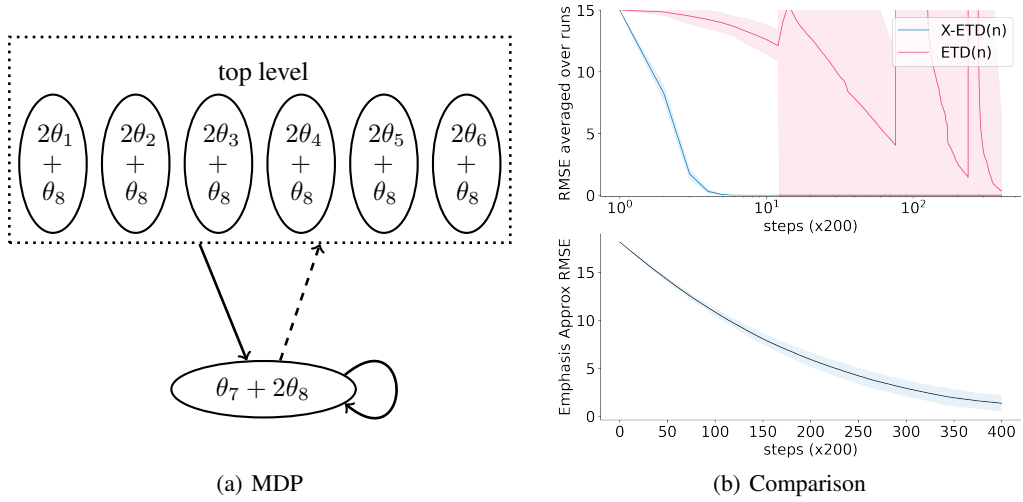
\begin{figure}[h!]
\centering
\subfigure[MDP]{
\begin{tikzpicture}[dgraph]
\node[ellip] (s1) at (0.,1) {$2\theta_1$\\$\,+$\\$\,\theta_8$};
\node[ellip] (s2) at (1.1,1) {$2\theta_2$\\$\,+$\\$\,\theta_8$};
\node[ellip] (s3) at (2.2,1) {$2\theta_3$\\$\,+$\\$\,\theta_8$};
\node[ellip] (s4) at (3.3,1) {$2\theta_4$\\$\,+$\\$\,\theta_8$};
\node[ellip] (s5) at (4.4,1) {$2\theta_5$\\$\,+$\\$\,\theta_8$};
\node[ellip] (s6) at (5.5,1) {$2\theta_6$\\$\,+$\\$\,\theta_8$};
\node[ellip, text width=1.5cm] (s7) at (2.8,-2) {$\,\,\theta_7+2\theta_8$};
\node[input] (ss) at (2.8, 2.2) {top level};
\node[system,fit=(ss) (s1) (s2) (s3) (s4) (s5) (s6)] (t0){};
\draw[](t0.-120)--(s7.120);
\draw[dashed](s7.60)--(t0.-60);
\draw[](s7.10)to[out=40, in=-40,looseness=8] (s7.-10);
\end{tikzpicture}
\label{fig:baird_mdp}
}
\hspace{.2cm}
\subfigure[Comparison]{
\includegraphics[width=0.45\columnwidth]{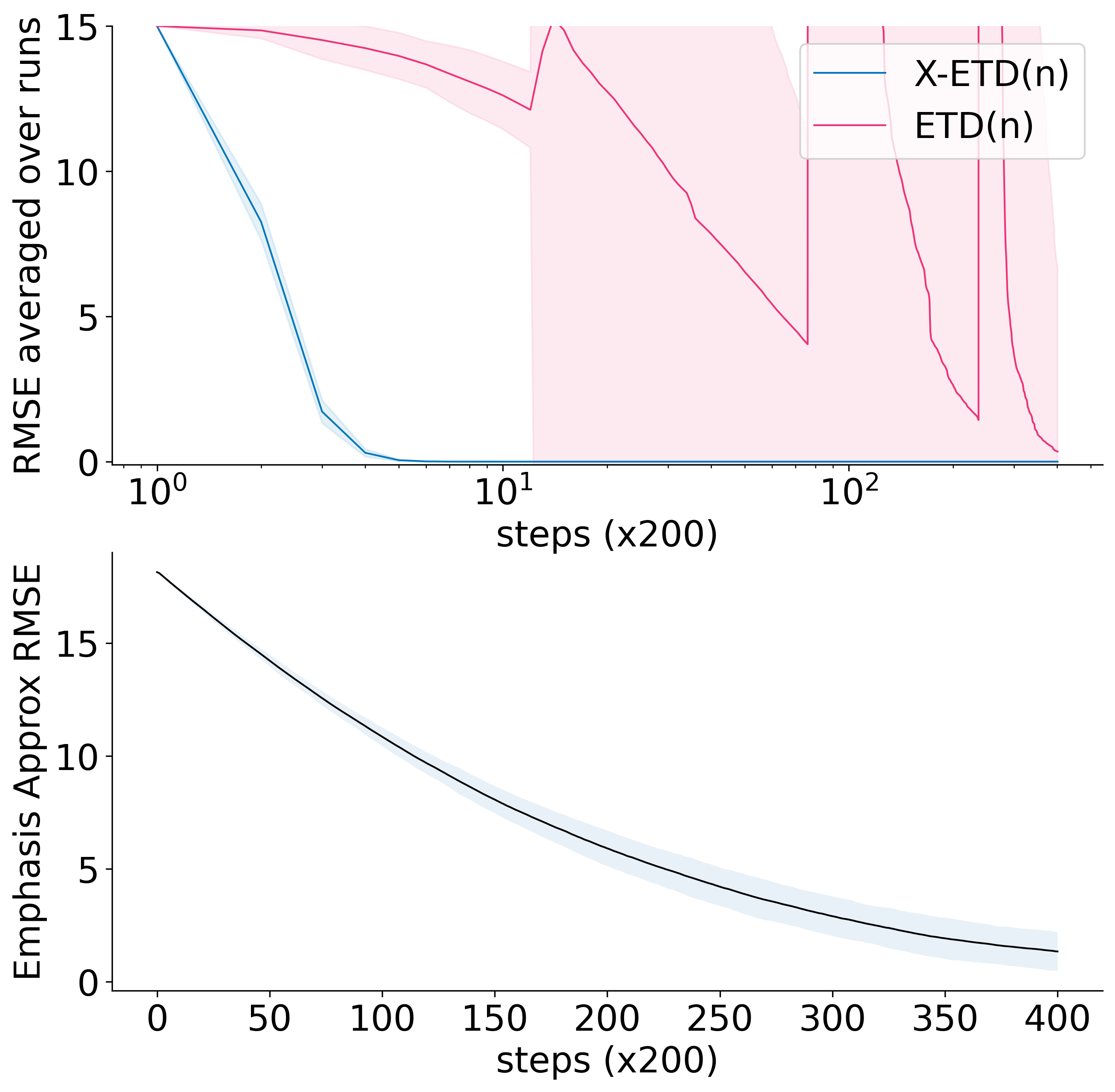}
\label{fig:baird_exp_appendix}
}
\caption{MDP illustration.
\subref{fig:baird_mdp}
Modified Baird's counterexample. Red lines indicate the target policy $\pi($solid$|\cdot)=0.3$, $\pi($dashed$|\cdot)=0.7$. Blue lines indicate the behavior policy $\mu($solid$|\cdot)=6/7, \mu($dashed$|\cdot)=1/7$. When action is ``dashed'', the agent goes to a random state on the top level. When action is ``solid'', the agent goes to the bottom state. 
\subref{fig:baird_exp_appendix}
RMSE in the value estimates and RMSE in expected trace approximation over time in a modified version of Baird's counterexample. We report the performance of each algorithm using the best performing hyperparameters (according to RMSE of the value function) from an extensive sweep (described in text). Shaded regions indicate two standard deviations of the mean performance computed from 100 independent runs.
}
\label{fig:baird_illustration}
\end{figure}

\begin{figure}[htb!]
\centering
\includegraphics[width=0.6\columnwidth]{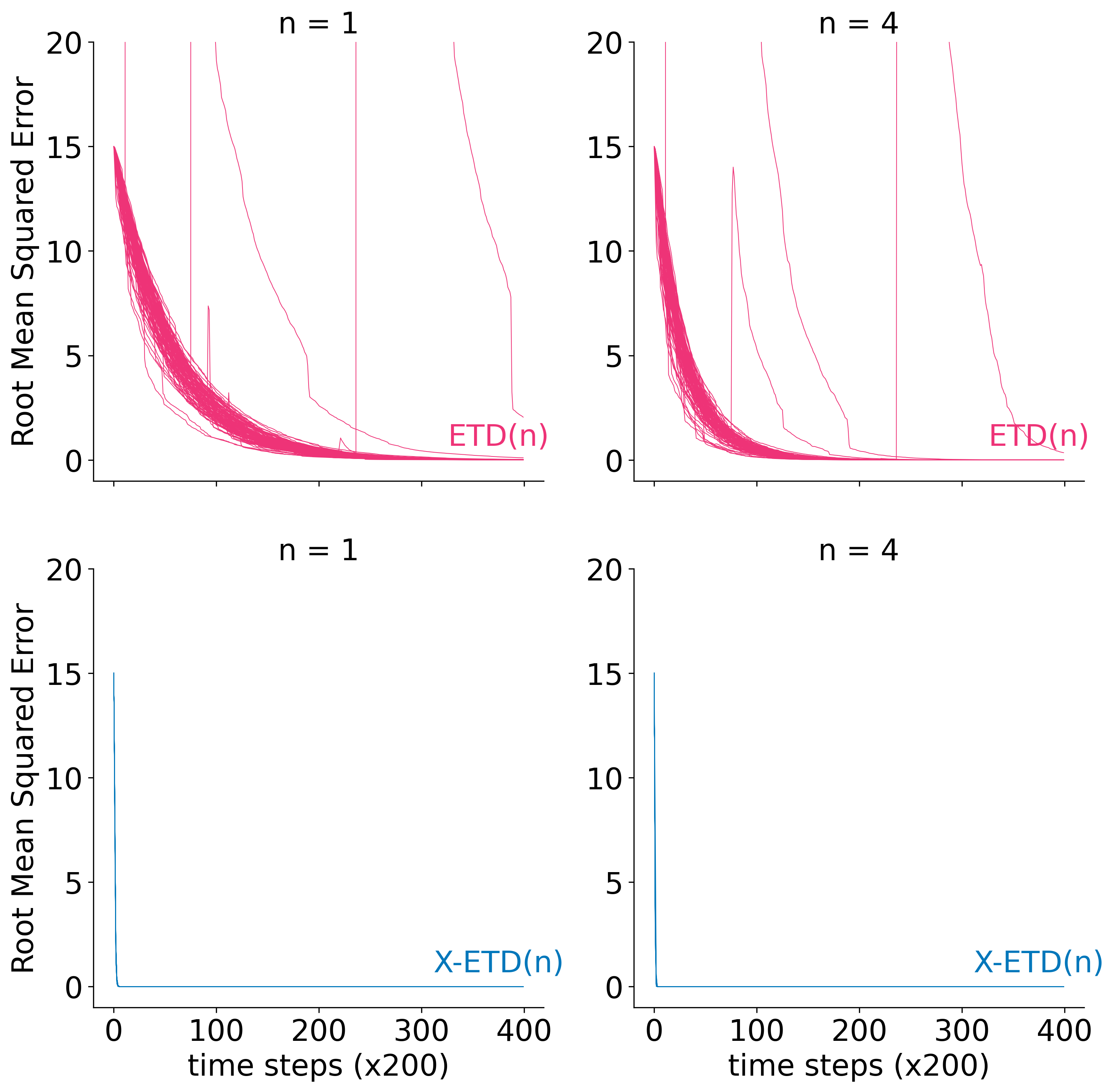}
\caption{RMSE in the value estimates over time in a modified version of Baird's counterexample. We plot each run individually to better characterize the performance of each algorithm.}
\vspace{-0.1cm}
\label{fig:app_baird_exp_ind}
\end{figure}
\end{document}